%% file: main.tex
\documentclass[letterpaper, 10 pt, journal, twoside]{IEEEtran}

\usepackage{amsmath}
\usepackage{amssymb}
\usepackage{amsthm}
\usepackage{xcolor}
\usepackage{subcaption}
\usepackage[pdftex]{graphicx}
\usepackage{algorithm}
\usepackage{algorithmicx}
\usepackage{algpseudocode}
\usepackage{pgf}
\usepackage{pgfplots}
\usepackage{float}

\usepackage[T1]{fontenc}%
\usepackage[utf8]{inputenc}%
\usepackage{textcomp}%
\usepackage{lastpage}%
\usepackage{graphicx}%
\usepackage{xcolor}%
\usepackage{colortbl}%
\usepackage{pgfplots}%
\usepackage{tikz}%
\usepackage{tabularx}%
\usepackage{multirow}%

\pgfplotsset{compat=1.18}
\algrenewcommand\algorithmicrequire{\textbf{Input:}}
\algrenewcommand\algorithmicensure{\textbf{Output:}}

\usepackage[
	activate   = {true},
	protrusion = false,
	expansion  = true,
	kerning    = true,
	spacing    = true,
	tracking   = false,
	auto       = true,
	selected   = true,
	factor     = 2000,
	stretch    = 50,
	shrink     = 20,
]{microtype}

\newtheorem{theorem}{Theorem}
\newtheorem{lemma}{Lemma} 
\newtheorem{definition}{Definition} 
\newtheorem{remark}{Remark}

\usepackage{xspace}

\usepackage{csquotes}
\usepackage[
maxbibnames=99,
maxcitenames=2,
natbib=true,
style=numeric-comp,
backend=biber,
sorting=none,
giveninits=true,
url=false, 
doi=false,
eprint=false,
isbn=false,
]{biblatex}

\definecolor{refcolor}{RGB}{215, 25, 28}
\definecolor{highlight}{RGB}{25, 215, 28}

\makeatletter
\let\NAT@parse\undefined
\makeatother
\usepackage[pdfa,colorlinks,bookmarksopen,bookmarksnumbered,allcolors=refcolor]{hyperref}
\usepackage[english]{babel}

\usepackage[nameinlink,capitalise]{cleveref}
\crefname{line}{line}{lines}
\crefname{figure}{Fig.}{Figs.}
\Crefname{figure}{Fig.}{Figs.}
\crefname{equation}{Eq.}{Eqs.}
\Crefname{equation}{Eq.}{Eqs.}
\crefname{section}{Sec.}{Secs.}
\Crefname{section}{Sec.}{Secs.}
\crefname{definition}{Def.}{Defs.}
\Crefname{definition}{Def.}{Defs.}
\crefname{algorithm}{Alg.}{Algs.}
\Crefname{algorithm}{Alg.}{Algs.}
\crefname{assumption}{Asm.}{Asms.}
\Crefname{assumption}{Asm.}{Asms.}
\crefname{theoren}{Thm.}{Thms.}
\Crefname{theorem}{Thm.}{Thms.}
\crefname{subassumption}{Asm.}{Asms.}
\Crefname{subassumption}{Asm.}{Asms.}
\crefname{table}{Tab.}{Tabs.}
\Crefname{table}{Tab.}{Tabs.}

\newcommand{\methodname}{StAC\xspace}
\newcommand{\expandmethodname}{Scheduling to Avoid Collisions\xspace}

\ifCLASSINFOpdf
\else
\fi

\newcommand{\Weihang}[1]{{\color{black}#1}}

\begin{document}

\title{Efficient Multi-Robot Motion Planning for Manifold-Constrained Manipulators by Randomized Scheduling and Informed Path Generation}

\author{Weihang Guo, Zachary Kingston, Kaiyu Hang, and Lydia E. Kavraki%
\thanks{Manuscript received: September 19, 2025; Revised November 9, 2025; Accepted January 5, 2026.}%Use only for final RAL version
\thanks{This paper was recommended for publication by
Editor Olivier Stasse upon evaluation of the Associate Editor and Reviewers comments. 
This work was supported in part by NSF CCF 2336612 and Enrich funds at Rice University.} %Use only for final RAL version
\thanks{WG, ZK, KH, and LEK are affiliated with the Department of Computer Science, Rice University, Houston TX, USA {\tt\small \{wg25, zak, kaiyu.hang, kavraki\}@rice.edu}. LEK and KH are also affiliated with the Ken Kennedy Institute at Rice University. ZK is now affiliated with the Department of Computer Science, Purdue University, West Lafayette IN, USA {\tt\small zkingston@purdue.edu}.}%
\thanks{Digital Object Identifier (DOI): see top of this page.}
}

\markboth{IEEE Robotics and Automation Letters. Preprint Version. Accepted January, 2026}
{Guo \MakeLowercase{\textit{et al.}}: StAC: Efficient Multi-Robot Motion Planning for Manifold-Constrained Manipulators} 

\maketitle 

\input{includes/00_Abstract}

% \begin{IEEEkeywords}
% IEEE, IEEEtran, journal, \LaTeX, paper, template.
% \end{IEEEkeywords}

\IEEEpeerreviewmaketitle

\input{includes/01_Introduction}

\input{includes/02_Related_Works}
\input{includes/04_CAC}

\input{includes/05_Experiment}

\input{includes/06_Conclusion}

\input{includes/07_Acknowledgements}

\ifCLASSOPTIONcaptionsoff
  \newpage
\fi

\printbibliography{}

\end{document}

%% file: includes/00_Abstract.tex
\begin{abstract}

Multi-robot motion planning for high degree-of-freedom manipulators in shared, constrained, and narrow spaces is a complex problem and essential for many scenarios such as construction, surgery, and more. Traditional coupled methods plan directly in the composite configuration space, which scales poorly; decoupled methods, on the other hand, plan separately for each robot but lack completeness. Hybrid methods that obtain paths from individual robots together require the enumeration of many paths before they can find valid composite solutions.
This paper introduces \expandmethodname~(\methodname), a hybrid approach that more effectively composes paths from individual robots by \emph{scheduling} (adding stops and coordination motion along all paths) and generates paths that are likely to be feasible by using bidirectional feedback between the scheduler and motion planner for informed sampling.
\methodname uses 10 to 100 times fewer paths from the low-level planner than state-of-the-art hybrid baselines on challenging problems in manipulator cases.

\end{abstract}

%% file: includes/01_Introduction.tex
\section{Introduction}

Multi-robot systems are essential for solving tasks beyond the capabilities of a single robot~\cite{khamis2015multirobot}.
Multi-Robot Motion Planning (MRMP) finds continuous, collision-free paths for multiple robots, considering collisions not just with obstacles but also between moving robots. 
In MRMP, the approaches can be categorized into three main types: decoupled, coupled, and hybrid methods. Decoupled methods, such as those based on velocity obstacles~\cite{van2011reciprocal} or priority-based planning~\cite{van2005prioritized}, are primarily applied to mobile and aerial robots. However, they are ill-suited for high-DoF manipulators operating in constrained, cluttered environments, where the mapping from workspace to configuration space is high-dimensional and complex, and the configuration space contains many narrow passages. Consequently, they often face deadlocks and are inefficient in capturing the required coordination. 
This paper focuses on the constrained and cluttered scenarios, and we consider robotic arms with manifold constraints, which limit their motion to a surface, such as a planar work surface~(see~\cref{fig:arm_doorway}). In such complex cases, coupled and hybrid methods are more widely used due to their ability to handle the intricate coordination required.

\input{figs/0_2arms_doorway_env}

Coupled methods plan directly in the composite configuration space by modeling the multi-robot system as a single entity. Algorithms such as dRRT~\cite{Solovey2013FindingAN} extend the principles of classic RRT~\cite{kuffner2000rrtconnect}, sampling composite configurations that jointly encode the states of all robots. dRRT often scales poorly, making it unsuitable for complex scenarios involving high-DoF manipulators in confined spaces~(see~\ref{sec:experiment}). 
Hybrid approaches~\cite{cbs, solis_vidana_representation-optimal_2021} combine the strengths of coupled and decoupled methods by using a low-level planner to generate paths individually and a high-level planner to schedule them. 
The process typically involves two main components. The first component is \emph{planning}: each robot plans a path from its start to its goal, avoiding obstacles and ignoring other robots. The second component is \emph{scheduling}: based on these planned trajectories, the execution is scheduled by adjusting velocities or introducing pauses to ensure that no inter-robot collisions occur during motion.

A less-explored aspect of hybrid approaches is effective scheduling. Methods such as CBSMP~\cite{solis_vidana_representation-optimal_2021} emphasize enumerating paths over exploring alternative schedules. \Weihang{CBSMP, for example, checks collisions only under a fixed schedule (typically simultaneous motion) and, upon conflict, immediately regenerates paths. This neglects many viable schedules, such as allowing one robot to move first, which could resolve conflicts without re-planning the individual robot trajectory.} In constrained settings like~\cref{fig:arm_doorway}, where only one arm can pass through a narrow doorway at a time, simultaneous motion forces long detours that may be infeasible. 
By contrast, effective scheduling increases the likelihood of finding solutions and offers richer feedback to low-level planners.

To this end, we introduce the \expandmethodname (\methodname) algorithm, a probabilistically complete hybrid approach designed to emphasize and explore the importance of scheduling in MRMP. This method is specifically tailored to high-DoF robots operating in manifold-constrained, narrow, cluttered, and shared workspaces.
\methodname's high-level scheduler generates randomized schedules that coordinate the robot's motions to avoid collisions for paths generated by individual path planners.
If the high-level planner cannot schedule a valid solution, collision information from all schedules is given to the low-level planners, who use this feedback to plan alternative paths. %
Our empirical results demonstrate that \methodname requires 10 to 100 times fewer paths from the low-level planner to find a solution than CBSMP, significantly reducing planning times in highly constrained scenarios where state-of-the-art hybrid methods fail to solve even once.

%% file: figs/0_2arms_doorway_env.tex
\begin{figure}[!ht]
    \centering
    \includegraphics[width=0.7\linewidth]{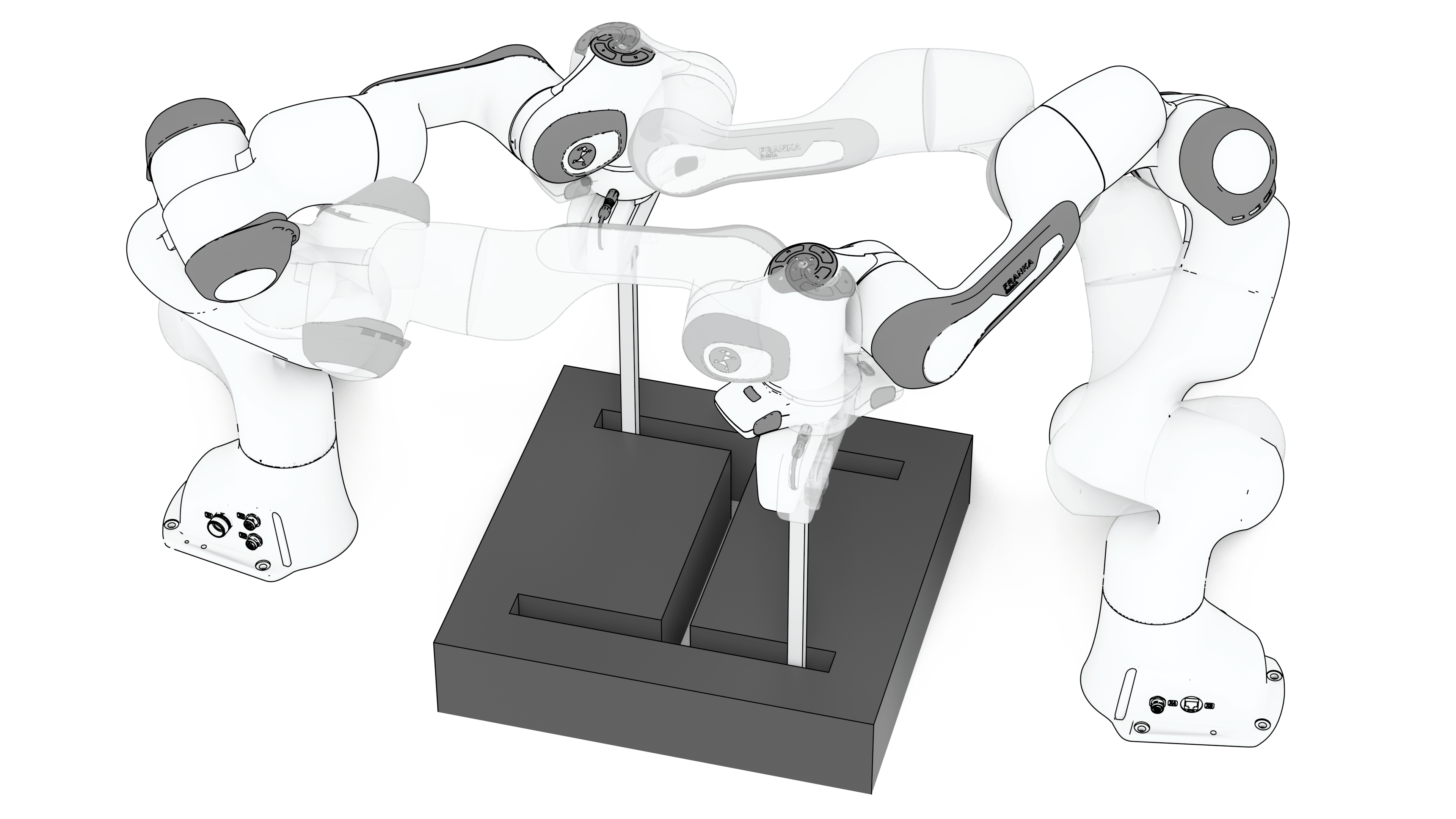}
    \caption{We extend the classic planar doorway problem to high-DoF manipulators. The end effectors of the two arms are constrained to a manifold on the same plane. In this scenario, two robots must swap positions of their end effector by navigating a narrow passage, which forces one to wait. As shown in \cref{sce:doorway}, the \methodname algorithm solves this challenge in an average of 5 seconds, whereas the state-of-the-art hybrid method fails to find a solution within 60 seconds.
}
    \label{fig:arm_doorway}
\end{figure}

%% file: includes/02_Related_Works.tex
\section{Preliminaries and Related Work}

The Multi-Robot Motion Planning (MRMP) problem involves determining feasible continuous paths for a set of robots \( A = \{a_1, \cdots, a_I\}\) operating in an environment \( W \).
The composite state space of all robot configurations is known as the configuration space $\mathbb{C}$, the Cartesian product of all individual robot configuration spaces $\mathbb{C}_1 \times \cdots \times \mathbb{C}_I$.
Here, $\mathbb{C}_i$ represents the configuration space for robot \( a_i \in A \).

We are interested in the case where each robot is \emph{task constrained}, e.g., they must keep their end-effector on a planar work surface.
We represent these constraints as \emph{manifold constraints} and use notation from~\cite{kingston2019exploring-implicit-spaces-for-constrained}: the constraint imposed on robot $a_i$ is defined as the function $f_i : \mathbb{C}_i \rightarrow \mathbb{R}^k$, which results in the implicit submanifold $\mathbb{M}_i \subset \mathbb{C}_i$ where $f_i(q) = \mathbf{0}, q \in \mathbb{M}_i$.
A review of methods for planning the motion of a single manifold-constrained robot is given in~\cite{kingston2018ar}.
The composite manifold of all constraint-satisfying configurations is given as $\mathbb{M} = \mathbb{M}_1 \times \cdots \times \mathbb{M}_I$.

The path of \( a_i \) is a continuous motion within a subset of $\mathbb{M}_i$ called the free space of \( a_i \), denoted as $\mathbb{M}_{free, i}$.
The composite constraint-satisfying free space $\mathbb{M}_{free}\subseteq \mathbb{M}$ consists of configurations in which the robots do not collide with obstacles or pairwise with each other, and satisfy constraints.
Each robot $a_i$ has an initial configuration \( c_{s,i} \) and a set of goals \( G_i \subset \mathbb{M}_{free, i} \).
The objective is to find a continuous path with a schedule \( \rho_i(s): [0, 1] \to \mathbb{M}_{free,i} \) for each robot $a_i$ that transitions it from \( c_{s,i} \) to \( G_i \) without colliding with any obstacles $o\in W$ and all other robots $a_j \in A\backslash\{a_i\}$.
Initially, this path is parameterized linearly over time, meaning $s(t) = \frac{t}{t_{\text{final}}}$ for $t \in [0, t_{\text{final}}]$, where  $t_{\text{final}}$ is the total time allocated for the robot to complete its motion from start to goal. To adjust the timing and coordination among robots, we introduce a scheduling function  $\sigma_i(t): [0, t_{\text{final}}] \to [0, 1]$ that modifies the linear mapping. The actual motion of the robot is then defined as  $\rho_i(t) = p_i(\sigma_i(t))$, where $\sigma_i(t)$ adjusts the progression along the path  $p_i$.

\subsection{Coupled Methods}
Sampling-based motion planners (e.g., RRT~\cite{kuffner2000rrtconnect}, PRM~\cite{508439}) can solve MRMP by treating the multi-robot system as one composite system.
Dedicated methods plan in factored representations, such as dRRT~\cite{Solovey2013FindingAN} and variants~\cite{shome_drrt_2020}, which efficiently sample from individual robots' configuration spaces and search in the composite space. dRRT uses a prioritization rule to schedule the robots to connect two vertices in the tree. SSSP~\cite{okumura2023quick} iteratively builds local search spaces over individual robot roadmaps; here, all robots are scheduled to move simultaneously. The planner in~\cite{mcbeth2023topoguidance} uses workspace topology and convex relaxation to coordinate more robots. However, these methods do not extend directly to high-DoF cases due to the costly workspace to configuration space transformations and the complexity of the composite space.

\subsection{Decoupled Methods}

Decoupled approaches address the exponential complexity of MRMP by decomposing the problem into many single-robot problems, e.g., by treating other robots as velocity obstacles and moving in simultaneously~\cite{van2011reciprocal, arul2021vrvo, pan2020augment_control}. The velocity obstacles are computationally costly to determine in the configuration space of a high-degree-of-freedom manipulator robot. Other algorithms, such as~\cite{van2005prioritized,li2022learningpc, orthey2024multilevel}, schedule robots in a prioritized order, where each robot plans in sequence and treats previously selected robots as dynamic obstacles. \Weihang{Resolving deadlocks in multi-robot motion planning using prioritized methods requires careful design of the prioritization strategy. Prior work~\cite{ma2019searching} has shown that fixed-order prioritization is incomplete for multi-robot planning. Deadlock-free prioritized methods, such as~\cite{ma2019searching, vcap2015prioritized, zhang2024d}, have only been demonstrated for mobile-robot settings, where the mapping between the workspace and configuration space is straightforward and solving the single-robot planning problem is trivial.} Many decoupled algorithms are difficult to generalize to high-DoF cases due to unknown workspace-to-configuration space projections and differing state representations for each manipulator. Many learning-based methods~\cite{li2021GNN, li2022learningpc} also require data and training and do not readily generalize. 

\subsection{Hybrid Methods}

Hybrid methods address MRMP by combining high- and low-level planners. The low-level planner generates paths for individual robots (or subsets), while the high-level planner schedules execution, checks feasibility, and provides conflict feedback to guide re-planning. 
Graph-based methods such as CBS~\cite{cbs} and LaCAM~\cite{okumura2023lacam} iteratively refine paths by adding constraints to avoid collisions, with CBS variants~\cite{barer2014ecbs, lim2022cbsb} improving efficiency to handle hundreds of robots. 
\Weihang{However, CBS and MA-CBS operate on grid-based state representations, whereas our problem is defined in a continuous state space.}
\Weihang{MMD~\cite{shaoul2024multi} incorporates inter-robot conflicts as conditions to the diffusion policy, which requires significant data to train.}

CBSMP~\cite{solis_vidana_representation-optimal_2021} extends CBS to continuous spaces by using roadmaps for low-level planning. Later work~\cite{mcbeth2023topoguidance} outperforms the CBSMP in a planar robot setup, and their topology-guided strategy does not directly extend to a multi-arm setup. We therefore regard CBSMP as the state-of-the-art method for multi-arm planning in narrow passages and adopt it as our baseline. \Weihang{The high-level planner of CBSMP discretizes paths into uniform time segments, advancing all agents simultaneously. If collisions occur, the paths are discarded and replanned. However, many possible schedules exist to time a set of robot paths; for example, consider cars moving on the road, where traffic lights can coordinate the cars by asking them to stop. Unlike \methodname, CBSMP only explores one possible scheduling function of the paths.}

\subsection{Conflict Resolution and Path Scheduling}

Given a set of robot paths, execution can be coordinated by velocity tuning~\cite{sanchez_using_2002, Kant1986TowardET, 988973}, which assigns each robot a velocity profile, or by priority-based search~\cite{van2005prioritized}, which executes robots in a fixed order. Some works~\cite{wiktor2014psw, okumura2023lacam} incorporate stopping actions during path search in \emph{discrete} spaces, but these strategies do not directly extend to continuous settings. In contrast, our scheduling explicitly reorders motion timing across robots, enabling more flexible coordination. The work in~\cite{solis2024arc} similarly reschedules around conflict areas, using priority ordering and reverting to a composite PRM in failure cases. Nevertheless, in high-DoF contexts, subproblems \emph{still} involve planning for multiple manipulators, which remains challenging. The method in~\cite{kazumi2022continuous_prioritized_interval} precomputes collision pairs with continuous time intervals to efficiently find collision-free schedules, but it applies only to 2D roadmaps where neighbor searches are efficient. Finally, \cite{okumura2022ctrm} uses a learning-based scheduling approach, which could also be integrated into the \methodname framework to further improve performance.

%% file: includes/04_CAC.tex
\section{\expandmethodname}

In this section, we introduce the core idea of \expandmethodname(\methodname) and elaborate on its details in the following subsections.

\subsection{Core Idea}\label{sec:core_idea}
\begin{algorithm}[!ht]
\caption{The \methodname Algorithm}
\footnotesize
    \begin{algorithmic}[1]
        \Require Set of robots $A$, Maximum reschedule attempts $N_{RA}$, \Weihang{Maximum replan attempts $N_{B}$}
        \Ensure Solution $S$
        \While{time available}
        \State Initial motion plans $P \leftarrow \emptyset$ 
        \For{each robot $a_i \in A$}
            \State $p_i\leftarrow a_i.\Call{QueryPRM}{\null}$ \label{line:ind_mp}
            \State $P\leftarrow P\cup \{p_i\}$
        \EndFor
        \State $n \gets 0$
        \While{$n \leq N_{B}$} \Comment{\Weihang{\cref{sec:retry}}}
            \State Attempt $i \gets 0$; Collision Record $record \leftarrow \emptyset$; $n \gets n+1$
            \While{$i \leq N_{RA}$}\Comment{Iterative schedule same set of paths} \label{line:begin_coord}
                \State Candidate Solution $S^* \gets \Call{Schedule}{P}$ \Comment{\cref{sec:coord}}
                \State Collision Information $x \gets \Call{CollisionCheck}{S^*}$
                \If{$x = \emptyset$}
                    \State \Return $S \gets S^*$ \Comment{Find the valid scheduling}
                \EndIf
                \State $\Call{RecordCollision}{x, record}$; $i \gets i+1$ \label{line:record_collision} \Comment{\cref{sec:feedback}}
            \EndWhile \label{line:end_coord}
            \State $P \gets \emptyset$ \Comment{Cannot find a valid scheduling.}
            \For{each robot $a_i \in A$}
                \State $a_i.\Call{Update}{p_i, record[i]}$\label{line:update_collision} \Comment{Update the feedback to robots}
                \State $p_i \gets a_i.\Call{QueryPRMWithExperience}{\null}$ \Comment{\cref{sec:agent}}
                \State $P\leftarrow P\cup \{p_i\}$
            \EndFor
        \EndWhile
    \EndWhile
    \State \Return $S \gets \emptyset$

    \end{algorithmic}
\label{alg:dCAB}
\end{algorithm}

\methodname consists of the scheduler and the individual low-level motion planner $a_i\in A$ for a set of $A$ robots. The pseudocode of \methodname is given in \cref{alg:dCAB}. Initially, each low-level planner $i$ plans a path $p_i(s)$ from its start to goal in its own manifold-constrained configuration space $\mathbb{M}_i$ by calling $a_i.\Call{QueryPRM}{\null}$, which assumes the robot $i$ is the only robot in the environment. 

\begin{remark}[Individual low-level motion planner]\label{remark:low_level}
Each robot constructs its own projection-based manifold-constrained PRM~\cite{kingston2019exploring-implicit-spaces-for-constrained}, assuming it is the only robot in the environment. A path, $p$, is then obtained as a sequence of vertices in this PRM, with its \textbf{length}, $|p|$, defined by the number of vertices. Throughout this paper, when we refer to the \textbf{path vertex} or \textbf{path edge} of a robot, we specifically mean the corresponding vertex or edge in the PRM.
\end{remark}

Given a set of paths $P$ returned by the all robots' low-level planners, the scheduler computes \emph{schedules} $\sigma_1(t), \ldots, \sigma_A(t)$ for each $p_i \in P$ such that no collisions occur between robots. This process, which we call \textit{coordination space scheduling}~(\cref{sec:coord}). As illustrated in \cref{fig:coord}, $\Call{Schedule}{P}$ first generates a candidate solution $S^*$ by inserting random pauses along individual robot paths and assigning a random priority order for their progression to the next vertex (\cref{sec:coord}). A collision check is then performed on $S^*$. If $S^*$ is collision-free with respect to all obstacles and other robots, it becomes the solution. Otherwise, the scheduler records all collision edges between robots in the \textit{Collision Record} (\cref{line:record_collision}). The scheduler then iteratively reschedules for a new candidate solution using the same set of paths until it reaches the maximum rescheduling attempts ($N_{RA}$) (\cref{sec:feedback}). 

If a valid solution is not found after $N_{RA}$ attempts, the collision counts of all edges stored in the Collision Record are provided to the robots as feedback (\cref{line:update_collision}). Each robot's low-level planner then stores this collision information in a \textit{Collision History}. Based on this experience, each robot plans a different set of paths, $P$, using $\Call{QueryPRMWithExperience}{\null}$, biased towards paths that will avoid collisions by favoring edges that have had fewer collisions with other robots in the past (\cref{sec:agent}).

\subsection{Scheduling in Coordination Space}\label{sec:coord}
\begin{figure}[!ht]
    \centering
    \includegraphics[width=0.8\linewidth]{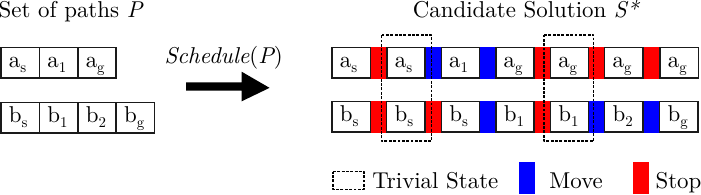}
    \caption{An illustration of \textsc{Schedule}(\textit{P}). Robots $a$ and $b$ are in the environment with their start $a_s, b_s$ and goal $a_g, b_g$. Paths returned by the low-level planners are obtained from the PRM (see~\cref{remark:low_level}). \textsc{Schedule}(\textit{P}) uses a set of individual robot paths to generate a candidate solution. The candidate solution uses the \emph{same set} of vertices as the original paths, but may repeat certain configurations to represent \emph{stops}. To transition from composite state $i$ to $i+1$, each robot will either \textcolor{blue}{move} or \textcolor{red}{stop}. If all robots \textcolor{red}{stop}, then state $i+1$ is \textit{trivial} and will be removed. }
    \label{fig:coord}
\end{figure}

The robot $i$'s path $p_i$ is a sequence of vertices, each corresponding to a robot configuration~(see~\cref{remark:low_level}) in $\mathbb{M}_{free,i}$. All vertices in path $p_i$ form a set $\mathbb{P}_i \subseteq \mathbb{M}_{free, i}$. The scheduler takes the set of paths $P$ from each robot's PRM as input. 
The goal of the scheduler is to find a continuous inter-robot collision-free paths schedule for all robots $\rho_i(t): [0, t_{final}] \rightarrow \mathbb{P}_i$ for $i \in \{1 \ldots I\}$. Thus, given any time $t \in [0, t_{final}]$, $\rho_i(t)$ represents robot $i$'s configuration $c_i \in \mathbb{P}_i$. By doing this, we constrain the range of individual scheduling functions to $\mathbb{P}_i$, which has a dimension of 1. The space of all scheduling functions for the composite system with a given set of robot paths $P$ is called the coordination space of $P$.
We used a sampling-based method to add \emph{random} stops and schedule the paths. \methodname samples stops between the vertices. 
We define $|p_i|$ equal to the number of vertices of path $p_i$. 
\begin{definition}[Candidate solution length]\label{def:sol_length}
Given a set of robot paths $P$, we define the length of a candidate solution as $L := \sum_{p_i \in P} |p_i|$, which equals the total number of vertices across all robot paths.
\end{definition}
We discretize the time into $L$ steps with all robots starting at their initial state at time $t_1 = 0$ and reaching their goal state at  $t_L = t_{\text{final}}$. From $t_i$ to $t_{i+1}$, the robot will either \textit{move} from its current vertex to the next vertex or \textit{stop}. We sample  $L - |p_i|$ stops for each robot to determine when it remains stationary. If all robots stop, then $t_{i+1}$ is a trivial state and will be removed. We call the path with a scheduling function a candidate solution $S^*$. The length of $S^*$, $|S^*|$, equals the number of non-trivial time steps. We provide an example in~\cref{fig:coord}.

\subsection{Collision Record for Robot Feedback}\label{sec:feedback}
\begin{figure}[!ht]
    \centering
    \includegraphics[width=0.8\linewidth]{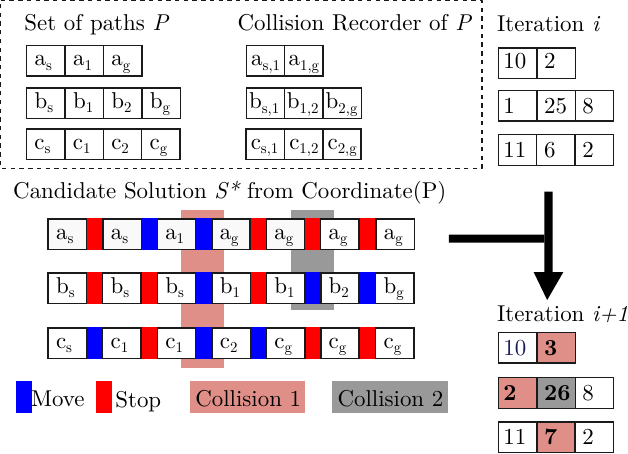}
    \caption{An example of the Collision Record during scheduling three robots. The Collision Record values for iteration $i$ are randomly initialized for illustration purposes. During each iteration, collision checking is performed on the candidate solution, and the counter for any edge involved in a collision is incremented by one. If a robot \textit{stops}, the collision is not counted. For example, a collision between \( a_{g} \) and \( b_{1,2} \) in the grey region only increments the counter for \( b_{1,2} \) from 25 to 26, since robot \( a \) is stopped.}
    
    \label{fig:collision_record}
\end{figure}

For the given set of paths $P$, we attempt to find a valid schedule within $N_{RA}$ times, as shown in~\cref{line:begin_coord}. We use a \emph{Collision Record} to log inter-robot collisions that happen during the $N_{RA}$ attempts. As shown in~\cref{fig:collision_record}, each path maintains a Collision Record with a length that matches the number of edges in the path. Whenever a collision occurs between edges, the counters for both edges are incremented by one. 
\Weihang{Unlike CBS-based planners, the Collision Record does not record the time $t$ at which a collision occurs. It stores only geometric information. This design effectively decouples scheduling from planning.}

\subsection{Motion Planning with Experience}\label{sec:agent}
\begin{figure}[!ht]
\vspace{10px}
    \centering
    \includegraphics[width=0.6\linewidth]{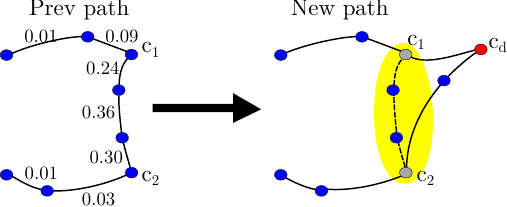}
    \caption{Low-level planners will plan a new path using the previous path in the Collision Record. The number near each edge is the normalized Collision Record of the edge and is randomly initialized for illustration purposes. The yellow region is the collision region determined by the random walk algorithm with $c_1$ and $c_2$ as its end vertices. $c_d$ is the detour configuration. The new path is shown in the solid line.}
    \label{fig:detour}
\end{figure}
If no solution is found by $\Call{Schedule}{P}$ after $N_{RA}$ iterations, \methodname sends the Collision Record back to each low-level planner. Each robot maintains a \emph{Collision History} list to store all paths and their corresponding Collision Records. The low-level planner uses this history as a reference to guide the selection of new paths.
\Weihang{The objective of each robot's low-level planner is to identify areas with a high likelihood of inter-robot collisions using the Collision History and to generate diverse paths to avoid these regions.} 

Specifically, the Collision History is initialized and expanded as follows. It is initialized with a \emph{root} element containing a null path and an empty Collision Record. \Weihang{All Collision Records from the scheduler are stored in the Collision History.} Each time $\Call{QueryPRMWithExperience}{\null}$ is called, the planner first selects an element from the Collision History based on a weighted sample. %
\Weihang{For an element $e$ in the Collision History, the cost of the element is assigned by a user-customized function $C(e)$. Here, we use $C(e) = l \cdot c(e) + k \cdot s(e)$, where $c(e)$ is the path cost in configuration-space distance, $s(e)$ is the number of times this node has already been selected.} Constants $l$ and $k$ adjust these contributions. \Weihang{Specifically, the sampler intends to select paths with smaller costs and fewer selections.} $C(e)$ is always greater than 0.

If the root element is selected, the robot will not identify any high-collision area and return $\Call{QueryPRM}{\null}$. If a \emph{non-root} element is selected, a random-walk algorithm is used to identify the high likelihood of inter-robot collisions region, represented as an interval along the element’s path.
\Weihang{The random-walk algorithm begins by selecting a random vertex from the element's Collision Record. From this vertex, the algorithm performs $k$ steps, visiting the left or right vertex with probabilities proportional to the neighboring edge's collision counts. For example, the probability of moving left is computed as the number of left-edge collisions divided by the sum of left-edge and right-edge collisions.
During the walk, the algorithm tracks the minimum and maximum value vertices visited. Those vertices are further mapped back to the selected element's PRM path, illustrated by the yellow region, $(c_1, c_2)$, in~\cref{fig:detour}. Because the walk is biased by the number of collisions, it preferentially selects regions with frequent inter-robot conflicts and ensures every interval retains a non-zero probability of being chosen.} 

The low-level planner reroutes around high-collision areas by randomly sampling a detour vertex $c_d$ in the PRM, which must be included in the path. As shown in~\cref{fig:detour}, it plans from $c_1$ to $c_d$ and $c_d$ to $c_2$, then merges these segments with the original path to form $(c_s, \ldots, c_1, \ldots, c_d, \ldots, c_2, \ldots, c_g)$, returned by $\Call{QueryPRMWithExperience}{\null}$.

\subsection{Re-schedule}\label{sec:retry}
As stated in~\cref{sec:feedback}, for a given set of paths $P$, we schedule the path maximum $N_{RA}$ times and record all inter-robot collisions. The parameter $N_{RA}$ allows users to trade off between scheduling and planning. When $N_{RA} = 1$, \methodname resembles CBSMP, where paths are scheduled only once. \Weihang{However, it differs from CBSMP in the following ways. (1) robot priority is determined randomly~(by adding stops) rather than pre-defined. (2) Robot paths are scheduled $N_{RA}$ times rather than only once. (3) Time at which collisions occur is decoupled from the low-level planner, making the planner purely geometric. (4) Conflicts are treated as soft constraints~(because of the random walk algorithm in~\cref{sec:agent}) instead of hard constraints as in CBS.}
Additionally, when the low-level planner calls $\Call{QueryPRMWithExperience}{\null}$, its PRM is frozen for efficiency, meaning no new vertices are added. This may hurt the probabilistic completeness of our algorithm. To address this, the low-level planner rebuilds the PRM and clears the Collision History after it has called $\Call{QueryPRMWithExperience}{\null}$ $N_B$ times.

\subsection{Probabilistic Completeness}
The proof of probabilistic completeness for the \methodname algorithm has two parts. First, we show that given a set of paths $P$, we eventually find a collision-free solution by calling $\Call{Schedule}{P}$ with infinitely many $N_{RA}$ if such a solution exists. Second, given infinitely many attempts for $\Call{QueryPRMWithExperience}{\null}$, the low-level planner of each robot will return all possible paths.

\begin{lemma}
    Let $P$ be a set of robot paths. If solution $S$ exist from scheduling robot paths $P$, then the length \Weihang{$|S| \leq \sum_{p_i \in P} (|p_i|-1)$}
\label{lem:finit_length}
\end{lemma}

\begin{proof}
If a solution $\dot{S}$ without trivial states had length $>\sum_{p_i \in P}(|p_i|-1)$, at least one robot moves per step. But the longest sequential solution has length $\sum_{p_i \in P}(|p_i|-1)$, a contradiction. Hence, $|S|\le\sum_{p_i \in P}(|p_i|-1)$.
\end{proof}

\begin{theorem}
Given an infinite number of reschedule attempts $N_{RA} \rightarrow \infty$, $\Call{Schedule}{P}$ will eventually try every possible valid scheduling of robot paths $P$. Consequently, if a solution $S$ exists in $P$, it will be found.

\label{the:cac_pc}
\end{theorem}
\begin{proof}
Assume, for contradiction, that a valid scheduling $S'$ of $P$ is never returned by $\Call{Schedule}{P}$, even as $N_{RA} \rightarrow \infty$. Let $l_i$ be the length of robot $i$'s path and $m = \sum_{p_i \in P} |p_i|$. There are $C(l_i, m)$ ways to schedule each path.

The probability of selecting $S'$ in one iteration is $\prod_{i \in A} \frac{1}{C(l_i, m)}$, so the probability it is not selected is $1 - \prod_{i \in A} \frac{1}{C(l_i, m)}$. Thus, as $N_{RA} \rightarrow \infty$, the probability that $S'$ is never selected approaches zero:
\begin{align*}
\lim_{N_{RA} \rightarrow \infty} P(\text{not } S') 
&= \lim_{N_{RA} \rightarrow \infty} \left(1 - \prod_{i \in A} \frac{1}{C(l_i, m)}\right)^{N_{RA}} \\
&= 0.
\end{align*}

This contradicts the assumption, so $\Call{Schedule}{P}$ eventually returns every valid scheduling of $P$.

\end{proof}

We now show that each robot has a non-zero chance of returning all different paths~(including those with loops), to compensate for the monotonicity of scheduling.

\begin{theorem}
Let robot $i$'s low-level planner be probabilistically complete. Then, as the number of iterations $n \rightarrow \infty$, the probability that the planner returns a path $\epsilon$-close to any valid path $p$ approaches 1.
\label{the:all_paths}
\end{theorem}

\begin{proof}
Fix any valid path $p = \{c_s, c_1, \ldots, c_g\}$ and any $\epsilon > 0$. By probabilistic completeness of the low-level planner, the probability of sampling a detour point within $\epsilon$ of $c_1$ is nonzero. Repeating this argument inductively for $c_2, c_3, \ldots, c_g$, each step of extending the path within $\epsilon$ has nonzero probability.

Let $q$ denote an $\epsilon$-close approximation of $p$. The probability that $q$ is not constructed in $n$ iterations is bounded by $(1-\delta)^n$, where $\delta > 0$ is the minimum probability of extending the path within $\epsilon$ at each step. As $n \to \infty$, $(1-\delta)^n \to 0$. 

Thus, for any valid path $p$, the probability that the planner produces an $\epsilon$-close path tends to 1 as $n \to \infty$. 
\end{proof}

\begin{theorem}
    Given an infinite number of iterations \( t \rightarrow \infty \) and an infinite maximum number of reschedule attempts $N_{RA}\rightarrow \infty$, \methodname will find a solution if one exists.
    \label{thm:cac_solution}
\end{theorem}

\begin{proof}
According to \cref{the:all_paths}, the probability that a robot returns a path $\epsilon$-close to any valid path, including those with loops, approaches 1 as $t \rightarrow \infty$. Consequently, \methodname will, with probability 1, eventually obtain $\epsilon$-close approximations of every possible combination of robot paths as $t \rightarrow \infty$.

    As stated in \cref{the:cac_pc}, for each possible scheduling of paths, \methodname will find a solution if one exists when \( t \rightarrow \infty \). Therefore, given infinite iterations, the \methodname is guaranteed to identify a solution if one exists.
\end{proof}

%% file: includes/05_Experiment.tex
\section{Experimental Results}\label{sec:experiment}

We test \methodname on three setups: (1) the extended 2D doorway problem with constrained Panda arms (\cref{sce:doorway}), (2) three-manipulator coordination tasks (\cref{sec:plus}), and (3) a cross maze with 2–5 arms to evaluate scalability (\cref{sce:scaleup}). All algorithms are implemented in C++ in OMPL~\cite{sucan2012the-open-motion-planning-library} with manifold constrained planning~\cite{kingston2019exploring-implicit-spaces-for-constrained} and tested on a PC with an Intel i5-8365 1.6GHz CPU and 8GB of RAM. We utilized the MuJoCo physics engine~\cite{mujoco} for collision detection \Weihang{in all experiments}. \Weihang{We measure the end-to-end solving time after MuJoCo and the OMPL problem definition have been initialized}. We used CBSMP~\cite{solis_vidana_representation-optimal_2021}, dRRT~\cite{Solovey2013FindingAN}, \Weihang{and RRT-Connect~\cite{kuffner2000rrtconnect}} as our comparison baseline. Since the source code for CBSMP and dRRT is not available, we implemented them in OMPL and added support for planning under manifold constraints. \Weihang{All robots are subject to end-effector constraints. To efficiently generate samples on the manifold, we first sample the end-effector pose and compute the corresponding configurations using inverse kinematics. During interpolation, the algorithm projects intermediate configurations onto the constraint manifold based on~\citet{kingston2019exploring-implicit-spaces-for-constrained}.} We benchmarked different map sizes and the number of expansions before connecting to the goal for dRRT, selecting the parameters with the best performance. Additionally, we used the same set of PRM parameters for both CBSMP and \methodname.
\Weihang{Unless otherwise stated, we set $N_B=10$ for \methodname in all experiments.}

\subsection{3D Doorway Setup}\label{sce:doorway}

This setup involves a 3D doorway setup where two 7-DoF Franka Panda arms, each with an end stick, are tasked with swapping their end effector positions as shown in \cref{fig:arm_doorway}. The arms' end effectors are constrained to move only on the xy-plane within the maze. Each algorithm was run 60 times with a 60-second timeout.  We evaluated the performance of dRRT, RRT-Connect, CBSMP, and \methodname, and the cumulative distribution function (CDF) of the solve times is shown in \cref{fig:arms_doorway_result}. \methodname demonstrates superior performance compared to other state-of-the-art algorithms, while CBSMP fails to solve the problem even once. This result further highlights the power of scheduling over path enumeration, as the robot has more possible configurations to avoid collisions. 
\Weihang{dRRT has the shorter average path length of $5.56$, whereas \methodname has a length of $7.03$.} 
\input{figs/1_arms_doorway_result}

\subsection{Three-manipulator Setup}\label{sec:plus}

\input{figs/three_arm_setup}

Three Franka Panda manipulators plan motions in clustered environments with two setups (\cref{fig:3arms_setup}): (1) all manipulators have stick-shaped end effectors constrained to planar translation, with one shorter than the others, and (2) two manipulators have stick-shaped end effectors constrained to the plane, and the third manipulator has a paddle-shaped end effector that can also rotate and translate in a central circular region. We test six start-goal pairs with a 60s timeout for Case 1 and nine pairs with a 120s timeout for Case 2, each repeated 50 times. CBSMP fails on one pair in Case 1 and two in Case 2, while \methodname with $N_{RA}=200$ achieves the highest success rate (\cref{tab:plus_result}). \Weihang{For case 1, dRRT has the shortest average path length of $0.57$, and CBSMP and StAC are comparable with $0.73$ and $0.83$. For case 2, dRRT has the shortest average path length of 0.76, and CBSMP and StAC have lengths of $1.10$ and $1.53$  respectively.}
\input{figs/table_all_new}

\subsection{Multi-manipulator Setups}\label{sce:scaleup}

To test scalability, we evaluate \methodname on a cross maze with end effectors constrained to the yz-plane and staggered along the x-axis so only the s-shaped tips collide (\cref{fig:scale_up_detail_maze}, \cref{tab:all_arm}). We generated eight start-goal pairs, each run ten times, yielding 80 trials per algorithm (dRRT, CBSMP, and \methodname). \Weihang{Among the 8 sampled queries, some instances are easier than others. For these easier instances, our coordinator queries each individual motion planner only once to obtain a valid path and finds a feasible schedule. This is why the solve time for the first quartile of the four-arm problem is 0.075s. In contrast, CBSMP must perform extensive re-planning before resolving conflicts; for the same case, it issues 31 motion-planning queries before finding a solution.} With 2-4 arms, \methodname finds schedules within about ten queries, but beyond 30 DoF, scheduling becomes the bottleneck.

\begin{figure}[!ht]
    \centering
    \includegraphics[width=0.4\linewidth]{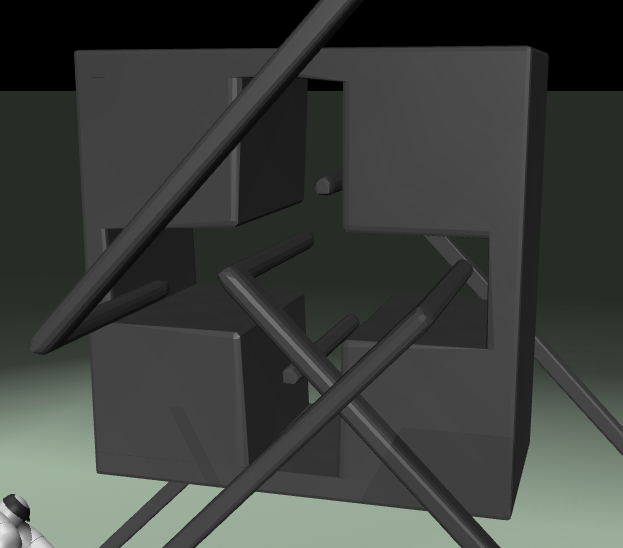}
    \caption{The multi-manipulator setups constrain the end effector to the yz-plane within the cross area. Each arm is staggered along the x-axis, so the ends of the S-shaped end effector will collide.}
    \label{fig:scale_up_detail_maze}
\end{figure}

\input{figs/table_all_arms}

\subsection{Insight and Lessons}

Our experiments show that incorporating feedback and path scheduling significantly improves multi-robot motion planning success rates and computational efficiency. By leveraging path scheduling, robots exploit paths better. \methodname queries fewer paths to find the solution.
\Weihang{CBSMP outperforms our method when the number of arms increases to five in~\cref{sce:scaleup}. We attribute this difference to how the two approaches scale in their scheduling components. CBSMP can struggle in tight environments, such as the 3D Doorway~(\cref{sce:doorway}) and Three-Manipulator problems~(\cref{sec:plus}), because resolving deadlocks often requires substantial expansion of the constraint tree. In contrast, our algorithm handles deadlocks more efficiently through repeated rescheduling and the use of soft constraints. However, in our method, the number of sampled stop actions grows linearly with the total path length across all robots, resulting in increased scheduling overhead as the team size grows. CBSMP’s scheduling effort, by comparison, is largely independent of the number of robots, since it evaluates only a single continuous-time execution without incorporating explicit per-robot stop actions.}

%% file: figs/1_arms_doorway_result.tex
\begin{figure}[!ht]
    \centering
    \resizebox{0.7\linewidth}{!}{\input{figs/arms_doorway_result}}
    \caption{CDF of solve time for the manipulator doorway setup with two 7-DOF Franka Panda arms shown in~\cref{fig:arm_doorway}. Note that CBSMP cannot solve this problem even once, as indicated by the horizontal line at 0. }
    \label{fig:arms_doorway_result}
\end{figure}

%% file: figs/arms_doorway_result.tex
\begin{tikzpicture}

\definecolor{color0}{rgb}{0.75,0,0.75}
\definecolor{color1}{rgb}{0.75,0.75,0}

\begin{axis}[
legend cell align={left},
legend style={
  fill opacity=0.8,
  draw opacity=1,
  text opacity=1,
  at={(0.97,0.03)},
  anchor=south east,
  draw=white!80!black
},
tick align=outside,
tick pos=left,
title={CDF of Solve Times for Different Planners},
x grid style={white!69.0196078431373!black},
xlabel={Time (s)},
xmajorgrids,
xmin=0, xmax=60,
xtick style={color=black},
y grid style={white!69.0196078431373!black},
ylabel={Cumulative Probability},
ymajorgrids,
ymin=0, ymax=1,
ytick style={color=black},
ytick={0,0.2,0.4,0.6,0.8,1},
yticklabels={0.0,0.2,0.4,0.6,0.8,1.0}
]
\addplot [line width=2.5pt, color0, dash pattern=on 1pt off 3pt on 3pt off 3pt]
table {%
0.309 0
1.183 0.0204081632653061
2.61 0.0408163265306122
2.943 0.0612244897959184
3.162 0.0816326530612245
3.839 0.102040816326531
4.951 0.122448979591837
5.093 0.142857142857143
6.225 0.163265306122449
6.812 0.183673469387755
9.551 0.204081632653061
10.617 0.224489795918367
14.606 0.244897959183673
17.158 0.26530612244898
18.764 0.285714285714286
20.481 0.306122448979592
24.159 0.326530612244898
24.977 0.346938775510204
27.389 0.36734693877551
27.532 0.387755102040816
27.755 0.408163265306122
28.526 0.428571428571429
28.558 0.448979591836735
29.694 0.469387755102041
32.414 0.489795918367347
33.803 0.510204081632653
36.09 0.530612244897959
37.866 0.551020408163265
42.568 0.571428571428571
43.646 0.591836734693878
44.465 0.612244897959184
47.819 0.63265306122449
51.728 0.653061224489796
55.374 0.673469387755102
56.753 0.693877551020408
60.001 0.714285714285714
60.003 0.73469387755102
60.012 0.755102040816326
60.033 0.775510204081633
60.038 0.795918367346939
60.074 0.816326530612245
60.082 0.836734693877551
60.122 0.857142857142857
60.15 0.877551020408163
60.151 0.897959183673469
60.158 0.918367346938776
60.19 0.938775510204082
60.193 0.959183673469388
60.24 0.979591836734694
60.371 1
};
\addlegendentry{\methodname, $N_{RA}$=1}
\addplot [line width=2.5pt, color0, dashed]
table {%
0.179 0
0.227 0.0169491525423729
0.238 0.0338983050847458
0.271 0.0508474576271186
0.369 0.0677966101694915
0.482 0.0847457627118644
0.526 0.101694915254237
0.561 0.11864406779661
0.59 0.135593220338983
0.649 0.152542372881356
0.668 0.169491525423729
0.687 0.186440677966102
0.701 0.203389830508475
0.715 0.220338983050847
0.72 0.23728813559322
0.76 0.254237288135593
0.846 0.271186440677966
0.852 0.288135593220339
1.038 0.305084745762712
1.377 0.322033898305085
1.447 0.338983050847458
1.5 0.355932203389831
1.701 0.372881355932203
1.825 0.389830508474576
1.867 0.406779661016949
1.971 0.423728813559322
2.461 0.440677966101695
2.536 0.457627118644068
2.561 0.474576271186441
2.751 0.491525423728814
2.888 0.508474576271186
3.019 0.525423728813559
3.049 0.542372881355932
3.159 0.559322033898305
3.281 0.576271186440678
3.499 0.593220338983051
3.598 0.610169491525424
3.602 0.627118644067797
3.661 0.644067796610169
3.885 0.661016949152542
3.904 0.677966101694915
4.443 0.694915254237288
4.561 0.711864406779661
5.513 0.728813559322034
6.748 0.745762711864407
6.775 0.76271186440678
6.788 0.779661016949153
7.859 0.796610169491525
8.219 0.813559322033898
8.506 0.830508474576271
8.681 0.847457627118644
8.687 0.864406779661017
12.265 0.88135593220339
15.893 0.898305084745763
16.97 0.915254237288136
18.487 0.932203389830508
19.392 0.949152542372881
22.056 0.966101694915254
22.81 0.983050847457627
28.832 1
};
\addlegendentry{\methodname, $N_{RA}$=200}
\addplot [line width=2.5pt, color0]
table {%
0.292 0
0.332 0.0169491525423729
0.353 0.0338983050847458
0.367 0.0508474576271186
0.393 0.0677966101694915
0.543 0.0847457627118644
0.604 0.101694915254237
0.709 0.11864406779661
0.711 0.135593220338983
0.746 0.152542372881356
0.837 0.169491525423729
1.093 0.186440677966102
1.101 0.203389830508475
1.331 0.220338983050847
1.458 0.23728813559322
1.865 0.254237288135593
2.012 0.271186440677966
2.061 0.288135593220339
2.064 0.305084745762712
2.109 0.322033898305085
2.109 0.338983050847458
2.157 0.355932203389831
2.347 0.372881355932203
2.447 0.389830508474576
2.448 0.406779661016949
2.497 0.423728813559322
2.661 0.440677966101695
2.788 0.457627118644068
2.843 0.474576271186441
3.131 0.491525423728814
3.26 0.508474576271186
3.411 0.525423728813559
3.433 0.542372881355932
3.582 0.559322033898305
3.6 0.576271186440678
3.638 0.593220338983051
3.728 0.610169491525424
3.864 0.627118644067797
4.094 0.644067796610169
4.513 0.661016949152542
5.019 0.677966101694915
5.1 0.694915254237288
5.352 0.711864406779661
5.865 0.728813559322034
6.082 0.745762711864407
6.152 0.76271186440678
6.529 0.779661016949153
6.967 0.796610169491525
7.342 0.813559322033898
7.597 0.830508474576271
8.626 0.847457627118644
8.762 0.864406779661017
8.91 0.88135593220339
9.909 0.898305084745763
10.157 0.915254237288136
10.809 0.932203389830508
12.057 0.949152542372881
16.947 0.966101694915254
21.324 0.983050847457627
26.678 1
};
\addlegendentry{\methodname, $N_{RA}$=400}
\addplot [line width=2.5pt, color1]
table {%
60 1
};
\addlegendentry{CBSMP}
\addplot [line width=2.5pt, blue]
table {%
3.575 0
4.205 0.0169491525423729
4.267 0.0338983050847458
4.488 0.0508474576271186
5.842 0.0677966101694915
6.647 0.0847457627118644
7.518 0.101694915254237
8.759 0.11864406779661
9.108 0.135593220338983
9.389 0.152542372881356
9.437 0.169491525423729
13.147 0.186440677966102
13.529 0.203389830508475
14.135 0.220338983050847
14.138 0.23728813559322
15.578 0.254237288135593
16.206 0.271186440677966
16.882 0.288135593220339
17.62 0.305084745762712
17.645 0.322033898305085
18.282 0.338983050847458
18.77 0.355932203389831
18.86 0.372881355932203
18.88 0.389830508474576
20.056 0.406779661016949
20.115 0.423728813559322
20.256 0.440677966101695
20.369 0.457627118644068
20.826 0.474576271186441
21.616 0.491525423728814
21.773 0.508474576271186
21.831 0.525423728813559
22.285 0.542372881355932
24.467 0.559322033898305
24.693 0.576271186440678
24.993 0.593220338983051
26.03 0.610169491525424
27.52 0.627118644067797
28.799 0.644067796610169
29.113 0.661016949152542
29.476 0.677966101694915
29.895 0.694915254237288
31.263 0.711864406779661
31.687 0.728813559322034
33.636 0.745762711864407
34.58 0.76271186440678
35.637 0.779661016949153
36.355 0.796610169491525
37.196 0.813559322033898
37.894 0.830508474576271
38.612 0.847457627118644
38.925 0.864406779661017
39.355 0.88135593220339
39.549 0.898305084745763
40.112 0.915254237288136
41.244 0.932203389830508
46.807 0.949152542372881
56.318 0.966101694915254
59.485 0.983050847457627
60.032 1
};
\addlegendentry{dRRT}
\addplot [line width=2.5pt, red]
table {%
9.58 0.0
9.986 0.01
13.029 0.02
14.885 0.03
15.939 0.04
16.345 0.05
18.858 0.06
19.422 0.07
19.984 0.08
20.14 0.09
20.996 0.1
21.603 0.11
21.621 0.12
22.445 0.13
24.123 0.14
24.188 0.15
24.337 0.16
25.638 0.17
25.696 0.18
25.959 0.19
26.099 0.2
26.463 0.21
26.806 0.22
26.833 0.23
28.65 0.24
28.991 0.25
29.117 0.26
29.375 0.27
30.105 0.28
30.143 0.29
30.331 0.3
30.402 0.31
30.574 0.32
30.658 0.33
30.863 0.34
31.524 0.35
31.755 0.36
32.94 0.37
33.33 0.38
33.612 0.39
35.416 0.4
35.643 0.41
35.899 0.42
36.058 0.43
36.996 0.44
37.427 0.45
37.527 0.46
37.773 0.47
38.007 0.48
39.107 0.49
40.072 0.5
40.341 0.51
41.238 0.52
41.567 0.53
42.067 0.54
42.095 0.55
42.314 0.56
44.878 0.57
44.953 0.58
45.897 0.59
46.347 0.6
46.449 0.61
46.763 0.62
47.83 0.63
48.829 0.64
49.585 0.65
49.65 0.66
49.709 0.67
50.617 0.68
51.396 0.69
51.961 0.7
51.971 0.71
52.552 0.72
53.496 0.73
53.839 0.74
54.138 0.75
54.774 0.76
56.588 0.77
56.619 0.78
57.456 0.79
58.256 0.8
58.626 0.81
59.204 0.82
59.214 0.83
60.213 0.84
60.324 0.85
60.356 0.86
60.358 0.87
60.396 0.88
60.397 0.89
60.406 0.9
60.415 0.91
60.427 0.92
60.467 0.93
60.47 0.94
60.475 0.95
60.489 0.96
60.579 0.97
60.612 0.98
60.725 0.99
60.861 1.0
};
\addlegendentry{RRTC}
\addplot [line width=2pt, color1, forget plot]
table {%
0 0
60 0
};
\addplot [line width=2.5pt, white!50.1960784313725!black, opacity=0.5, dashed, forget plot]
table {%
0 0
60 0
};
\addplot [line width=2.5pt, white!50.1960784313725!black, opacity=0.5, dashed, forget plot]
table {%
0 0.1
60 0.1
};
\addplot [line width=2.5pt, white!50.1960784313725!black, opacity=0.5, dashed, forget plot]
table {%
0 0.2
60 0.2
};
\addplot [line width=2.5pt, white!50.1960784313725!black, opacity=0.5, dashed, forget plot]
table {%
0 0.3
60 0.3
};
\addplot [line width=2.5pt, white!50.1960784313725!black, opacity=0.5, dashed, forget plot]
table {%
0 0.4
60 0.4
};
\addplot [line width=2.5pt, white!50.1960784313725!black, opacity=0.5, dashed, forget plot]
table {%
0 0.5
60 0.5
};
\addplot [line width=2.5pt, white!50.1960784313725!black, opacity=0.5, dashed, forget plot]
table {%
0 0.6
60 0.6
};
\addplot [line width=2.5pt, white!50.1960784313725!black, opacity=0.5, dashed, forget plot]
table {%
0 0.7
60 0.7
};
\addplot [line width=2.5pt, white!50.1960784313725!black, opacity=0.5, dashed, forget plot]
table {%
0 0.8
60 0.8
};
\addplot [line width=2.5pt, white!50.1960784313725!black, opacity=0.5, dashed, forget plot]
table {%
0 0.9
60 0.9
};
\addplot [line width=2.5pt, white!50.1960784313725!black, opacity=0.5, dashed, forget plot]
table {%
0 1
60 1
};
\end{axis}

\end{tikzpicture}

%% file: figs/three_arm_setup.tex
\begin{figure}[!ht]
\vspace{0.1in}
    \centering
    \begin{subfigure}[t]{0.17\textwidth}
        \centering
        \includegraphics[width=\textwidth]{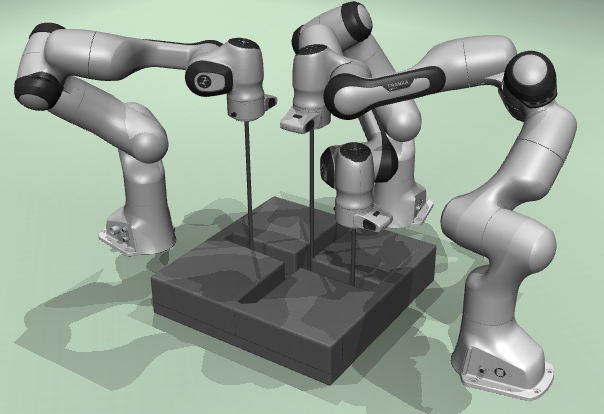}
        \caption{Case 1: 3-manipulator setup with a cross maze}
        \label{fig:3arms_case1}
    \end{subfigure}
    \begin{subfigure}[t]{0.3\textwidth}
        \centering
        \includegraphics[width=0.5\textwidth]{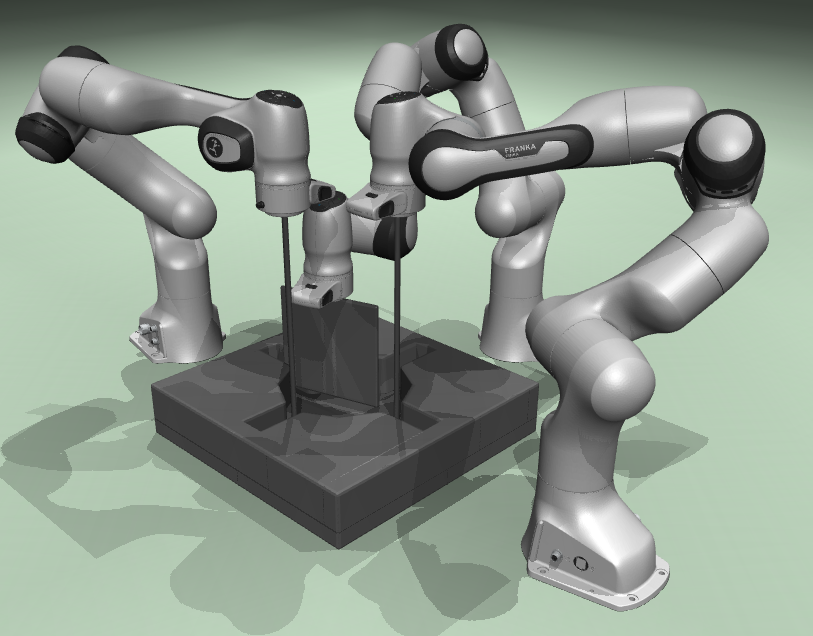}\includegraphics[width=0.4\textwidth]{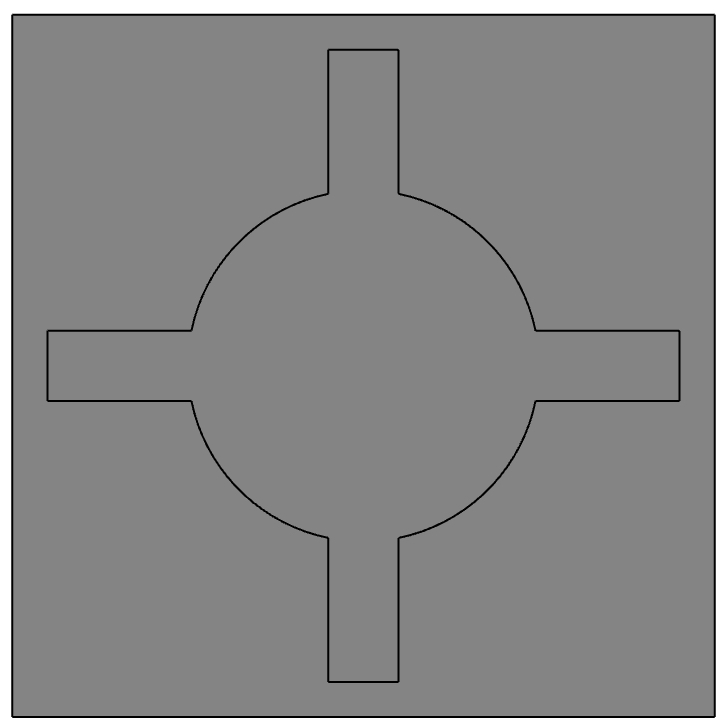} 
        \caption{Case 2: 3-manipulator setup with circular arrangement. Maze top view is shown on the right. }
        \label{fig:3arms_case2}
    \end{subfigure}
    \caption{Two different three-manipulator setups}
    \label{fig:3arms_setup}
\end{figure}

%% file: figs/table_all_new.tex
\definecolor{mycolor}{RGB}{0,0,0}%
\definecolor{lightgray}{gray}{0.9}%
\definecolor{darkgrey}{gray}{0.7}%
\definecolor{black}{rgb}{0, 0, 0}%
\normalsize%
\begin{table}[!ht]%
\vspace{10px}
\setlength{\tabcolsep}{5pt} %
\begin{tabularx}{\textwidth}{p{1cm}|c|r|c|c|c|c|c}%
\multicolumn{1}{c|}{Problem} & \multicolumn{2}{c|}{Method} & \multicolumn{3}{c|}{\# of Queries} & Coord. & Succ. \\%
& Planner & $N_{RA}$ & Q1 & Median & Q3 & ratio & \\%
\cline{1-8}%
\multirow{6}{=}{Case 1 (\cref{fig:3arms_case1})} & \cellcolor{lightgray}CBSMP & \cellcolor{lightgray}- & \cellcolor{lightgray}1 & \cellcolor{lightgray}1 & \cellcolor{lightgray}26 & \cellcolor{lightgray}11.5\% & \cellcolor{lightgray}79.3\% \\%
\arrayrulecolor{darkgrey}\cline{2-8}\arrayrulecolor{black}%
& dRRT & - & - & - & - & - & 75.7\% \\%
\arrayrulecolor{darkgrey}\cline{2-8}\arrayrulecolor{black}%
& \cellcolor{lightgray}RRTC & \cellcolor{lightgray}- & \cellcolor{lightgray}- & \cellcolor{lightgray}- & \cellcolor{lightgray}- & \cellcolor{lightgray}- & \cellcolor{lightgray}28.9\% \\%
\cline{2-8}%
& \multirow{3}{*}{\textcolor{mycolor}{\methodname}} & \textcolor{mycolor}{1} & 1 & 1 & 29 & 20.3\% & 86.0\% \\%
&& \cellcolor{lightgray}\textcolor{mycolor}{200} & \cellcolor{lightgray}1 & \cellcolor{lightgray}1 & \cellcolor{lightgray}10 & \cellcolor{lightgray}39.1\% & \cellcolor{lightgray}\textbf{89.7\%} \\%
&& \textcolor{mycolor}{400} & 1 & 1 & 10 & 37.5\% & 86.3\% \\%
\cline{1-8}%
\multirow{6}{=}{Case 2 (\cref{fig:3arms_case2})} & CBSMP & - & 29 & 73 & 282 & 1.4\% & 42.0\% \\%
\arrayrulecolor{darkgrey}\cline{2-8}\arrayrulecolor{black}%
& \cellcolor{lightgray}dRRT & \cellcolor{lightgray}- & \cellcolor{lightgray}- & \cellcolor{lightgray}- & \cellcolor{lightgray}- & \cellcolor{lightgray}- & \cellcolor{lightgray}19.0\% \\%
\arrayrulecolor{darkgrey}\cline{2-8}\arrayrulecolor{black}%
& RRTC & - & - & - & - & - & 22.2\% \\%
\cline{2-8}%
& \multirow{3}{*}{\textcolor{mycolor}{\methodname}} & \cellcolor{lightgray}\textcolor{mycolor}{1} & \cellcolor{lightgray}32 & \cellcolor{lightgray}174 & \cellcolor{lightgray}270 & \cellcolor{lightgray}15.4\% & \cellcolor{lightgray}47.3\% \\%
&& \textcolor{mycolor}{200} & 3 & 32 & 99 & 54.2\% & \textbf{64.0\%} \\%
&& \cellcolor{lightgray}\textcolor{mycolor}{400} & \cellcolor{lightgray}2 & \cellcolor{lightgray}32 & \cellcolor{lightgray}78 & \cellcolor{lightgray}65.6\% & \cellcolor{lightgray}61.7\% \\%
\cline{1-8}%
\end{tabularx}%
\caption{We randomly generated 9 start-goal pairs for each problem and tested them on the three algorithms, with each start-goal pair run 50 times and a timeout~(\textit{T/O}) of 60 seconds. For both problems, there are three subproblems where none of the planners could find a solution even once, so we excluded those subproblems for clearer visualization. We show the \# of path sets the high-level planner queried from the low-level planner before finding the solution for the first quartile (Q1), median, and third quartile. \Weihang{\textit{Coord. ratio} is the ratio between the scheduling time and the total time.}
}%
\label{tab:plus_result}%
\end{table}

%% file: figs/table_all_arms.tex
\definecolor{mycolor}{RGB}{0,0,0}%
\definecolor{lightgray}{gray}{0.9}%
\definecolor{darkgrey}{gray}{0.7}%
\definecolor{black}{rgb}{0, 0, 0}%
\begin{table*}[!ht]%
\vspace{10px}
\centering
\begin{tabularx}{\textwidth}{c|X|r|cc|cc|cc|c|c}%
\multicolumn{1}{c|}{Problem}&\multicolumn{2}{c}{Method}&\multicolumn{6}{c|}{Planning Statistics}&\multirow{1}{*}{Coord.}&\multirow{1}{*}{Succ.}\\%
&&&\multicolumn{2}{c}{Q1}&\multicolumn{2}{c}{Median}&\multicolumn{2}{c|}{Q3}&ratio&\\%
&Planner&$N_{RA}$&T&Q&T&Q&T&Q&&\\%
\cline{1-11}%
\multirow{5}{*}{\includegraphics[width=0.13\linewidth]{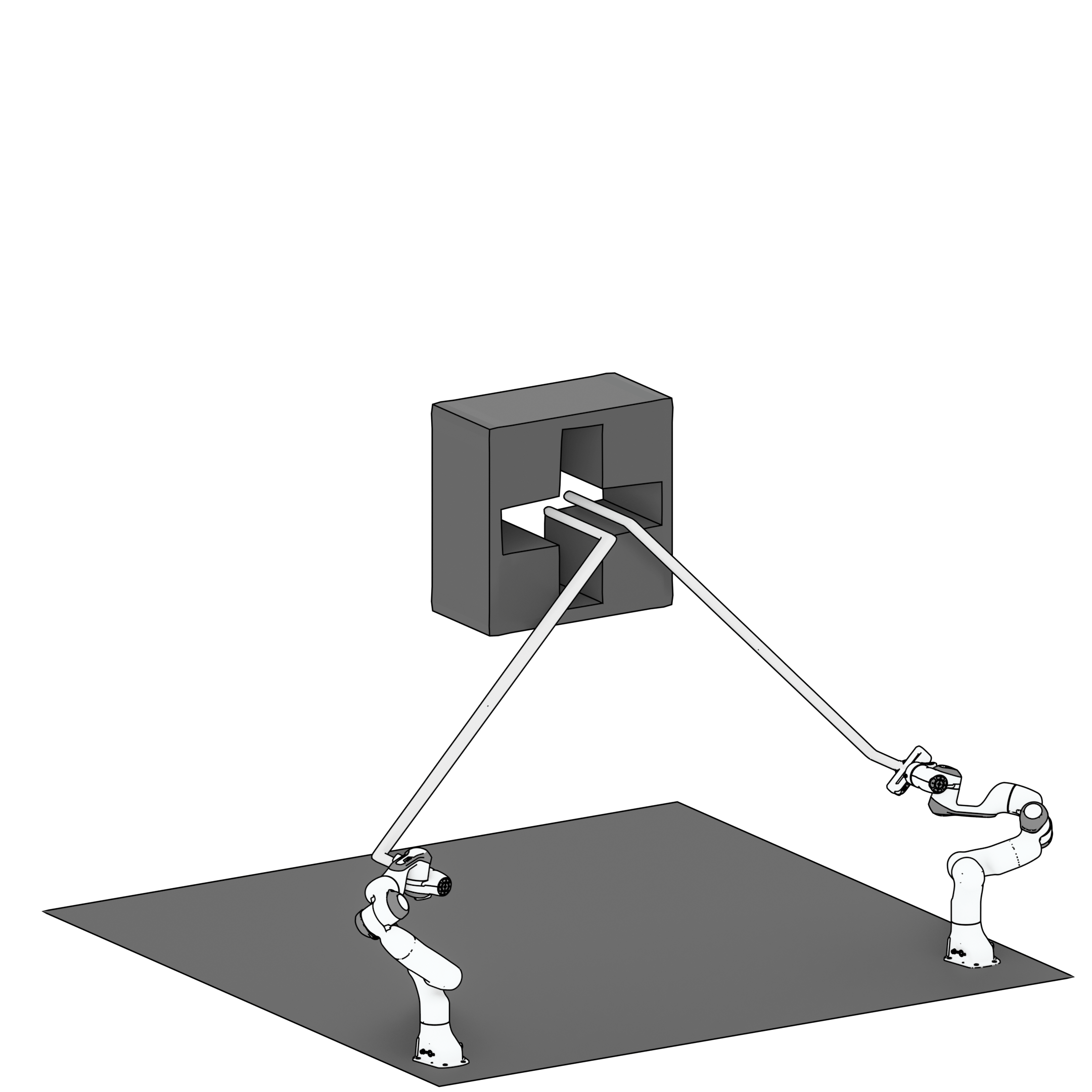}}&dRRT&-&0.016&-&0.042&-&0.068&-&-&\textbf{100.0\%}\\%
&\cellcolor{lightgray}CBSMP&\cellcolor{lightgray}-&\cellcolor{lightgray}0.015&\cellcolor{lightgray}1&\cellcolor{lightgray}0.099&\cellcolor{lightgray}14&\cellcolor{lightgray}0.160&\cellcolor{lightgray}11&\cellcolor{lightgray}10.3\%&\textbf{\cellcolor{lightgray}100.0\%}\\%
\arrayrulecolor{darkgrey}\cline{2-11}\arrayrulecolor{black}%
&\multirow{4}{*}{\textcolor{mycolor}{\methodname}}&\textcolor{mycolor}{1}&\textcolor{mycolor}{0.015}&\textcolor{mycolor}{1}&\textcolor{mycolor}{0.021}&\textcolor{mycolor}{1}&\textcolor{mycolor}{0.029}&\textcolor{mycolor}{1}&\textcolor{mycolor}{38.6\%}&\textbf{\textcolor{mycolor}{100.0\%}}\\%
&&\textcolor{mycolor}{\cellcolor{lightgray}200}&\textbf{\textcolor{mycolor}{\cellcolor{lightgray}0.013}}&\textcolor{mycolor}{\cellcolor{lightgray}1}&\textbf{\textcolor{mycolor}{\cellcolor{lightgray}0.017}}&\textcolor{mycolor}{\cellcolor{lightgray}1}&\textcolor{mycolor}{\cellcolor{lightgray}0.023}&\textcolor{mycolor}{\cellcolor{lightgray}1}&\textcolor{mycolor}{\cellcolor{lightgray}36.5\%}&\textbf{\textcolor{mycolor}{\cellcolor{lightgray}100.0\%}}\\%
&&\textcolor{mycolor}{400}&\textcolor{mycolor}{0.014}&\textcolor{mycolor}{1}&\textbf{\textcolor{mycolor}{0.017}}&\textcolor{mycolor}{1}&\textbf{\textcolor{mycolor}{0.022}}&\textcolor{mycolor}{1}&\textcolor{mycolor}{34.0\%}&\textbf{\textcolor{mycolor}{100.0\%}}\\%
&&\textcolor{mycolor}{\cellcolor{lightgray}400*}&\textcolor{mycolor}{\cellcolor{lightgray}0.015}&\textcolor{mycolor}{\cellcolor{lightgray}1}&\textcolor{mycolor}{\cellcolor{lightgray}0.018}&\textcolor{mycolor}{\cellcolor{lightgray}1}&\textcolor{mycolor}{\cellcolor{lightgray}0.028}&\textcolor{mycolor}{\cellcolor{lightgray}1}&\textcolor{mycolor}{\cellcolor{lightgray}4.3\%}&\textbf{\textcolor{mycolor}{\cellcolor{lightgray}100.0\%}}\\%
\cline{1-11}%
\multirow{5}{*}{\includegraphics[width=0.13\linewidth]{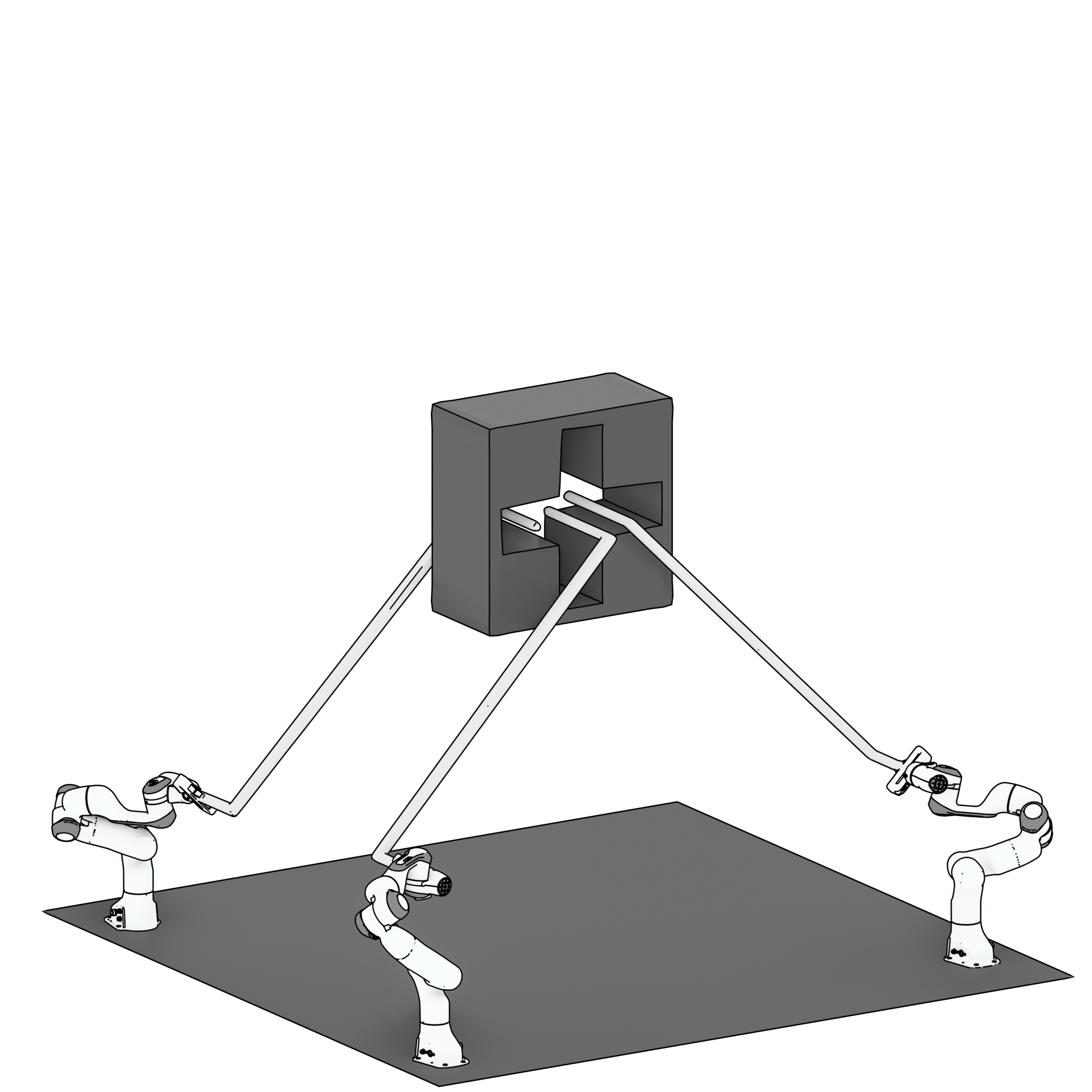}}&dRRT&-&1.398&-&2.794&-&5.242&-&-&\textbf{100.0\%}\\%
&\cellcolor{lightgray}CBSMP&\cellcolor{lightgray}-&\cellcolor{lightgray}0.996&\cellcolor{lightgray}31&\cellcolor{lightgray}2.906&\cellcolor{lightgray}145&\cellcolor{lightgray}8.162&\cellcolor{lightgray}315&\cellcolor{lightgray}3.1\%&\cellcolor{lightgray}94.4\%\\%
\arrayrulecolor{darkgrey}\cline{2-11}\arrayrulecolor{black}%
&\multirow{4}{*}{\textcolor{mycolor}{\methodname}}&\textcolor{mycolor}{1}&\textcolor{mycolor}{0.086}&\textcolor{mycolor}{1}&\textcolor{mycolor}{0.737}&\textcolor{mycolor}{8}&\textcolor{mycolor}{2.459}&\textcolor{mycolor}{22}&\textcolor{mycolor}{35.1\%}&\textbf{\textcolor{mycolor}{100.0\%}}\\%
&&\textcolor{mycolor}{\cellcolor{lightgray}200}&\textbf{\textcolor{mycolor}{\cellcolor{lightgray}0.057}}&\textcolor{mycolor}{\cellcolor{lightgray}1}&\textbf{\textcolor{mycolor}{\cellcolor{lightgray}0.085}}&\textcolor{mycolor}{\cellcolor{lightgray}1}&\textbf{\textcolor{mycolor}{\cellcolor{lightgray}0.177}}&\textcolor{mycolor}{\cellcolor{lightgray}1}&\textbf{\textcolor{mycolor}{\cellcolor{lightgray}84.7\%}}&\textbf{\textcolor{mycolor}{\cellcolor{lightgray}100.0\%}}\\%
&&\textcolor{mycolor}{400}&\textcolor{mycolor}{0.061}&\textcolor{mycolor}{1}&\textcolor{mycolor}{0.092}&\textcolor{mycolor}{1}&\textcolor{mycolor}{0.192}&\textcolor{mycolor}{1}&\textcolor{mycolor}{87.2\%}&\textbf{\textcolor{mycolor}{100.0\%}}\\%
&&\textcolor{mycolor}{\cellcolor{lightgray}400*}&\textcolor{mycolor}{\cellcolor{lightgray}0.058}&\textcolor{mycolor}{\cellcolor{lightgray}1}&\textcolor{mycolor}{\cellcolor{lightgray}0.098}&\textcolor{mycolor}{\cellcolor{lightgray}1}&\textcolor{mycolor}{\cellcolor{lightgray}0.249}&\textcolor{mycolor}{\cellcolor{lightgray}1}&\textcolor{mycolor}{\cellcolor{lightgray}91.3\%}&\textbf{\textcolor{mycolor}{\cellcolor{lightgray}100.0\%}}\\%
\cline{1-11}%
\multirow{5}{*}{\includegraphics[width=0.13\linewidth]{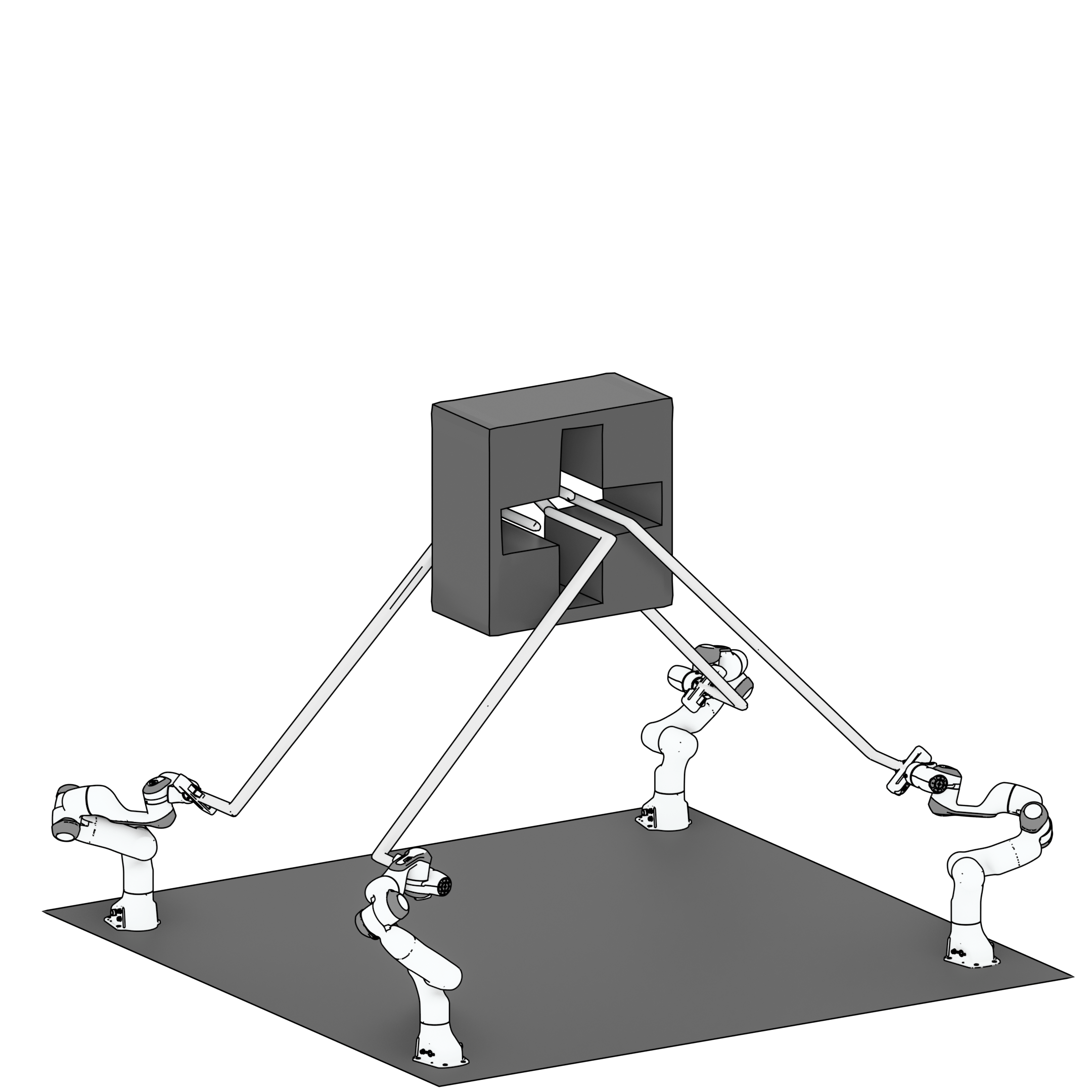}}&dRRT&-&7.926&-&37.143&-&T/O&-&-&67.1\%\\%
&\cellcolor{lightgray}CBSMP&\cellcolor{lightgray}-&\cellcolor{lightgray}0.524&\cellcolor{lightgray}31&\cellcolor{lightgray}13.009&\cellcolor{lightgray}516&\cellcolor{lightgray}T/O&\cellcolor{lightgray}2041&\cellcolor{lightgray}5.5\%&\cellcolor{lightgray}68.9\%\\%
\arrayrulecolor{darkgrey}\cline{2-11}\arrayrulecolor{black}%
&\multirow{4}{*}{\textcolor{mycolor}{\methodname}}&\textcolor{mycolor}{1}&\textcolor{mycolor}{0.561}&\textcolor{mycolor}{6}&\textcolor{mycolor}{5.078}&\textcolor{mycolor}{60}&\textcolor{mycolor}{T/O}&\textcolor{mycolor}{412}&\textcolor{mycolor}{58.6\%}&\textcolor{mycolor}{72.2\%}\\%
&&\textcolor{mycolor}{\cellcolor{lightgray}200}&\textbf{\textcolor{mycolor}{\cellcolor{lightgray}0.075}}&\textcolor{mycolor}{\cellcolor{lightgray}1}&\textcolor{mycolor}{\cellcolor{lightgray}0.621}&\textcolor{mycolor}{\cellcolor{lightgray}2}&\textcolor{mycolor}{\cellcolor{lightgray}28.174}&\textcolor{mycolor}{\cellcolor{lightgray}26}&\textcolor{mycolor}{\cellcolor{lightgray}92.7\%}&\cellcolor{lightgray}82.2\%\\%
&&\textcolor{mycolor}{400}&\textcolor{mycolor}{0.098}&\textcolor{mycolor}{1}&\textcolor{mycolor}{1.326}&\textcolor{mycolor}{3}&\textbf{\textcolor{mycolor}{23.586}}&\textcolor{mycolor}{20}&\textcolor{mycolor}{97.6\%}&\textbf{\textcolor{mycolor}{84.4\%}}\\%
&&\textcolor{mycolor}{\cellcolor{lightgray}400*}&\textcolor{mycolor}{\cellcolor{lightgray}0.083}&\textcolor{mycolor}{\cellcolor{lightgray}1}&\textbf{\textcolor{mycolor}{\cellcolor{lightgray}0.313}}&\textcolor{mycolor}{\cellcolor{lightgray}1}&\cellcolor{lightgray}T/O&\cellcolor{lightgray}17&\cellcolor{lightgray}94.6\%&\cellcolor{lightgray}73.3\%\\%
\cline{1-11}%
\multirow{5}{*}{\includegraphics[width=0.13\linewidth]{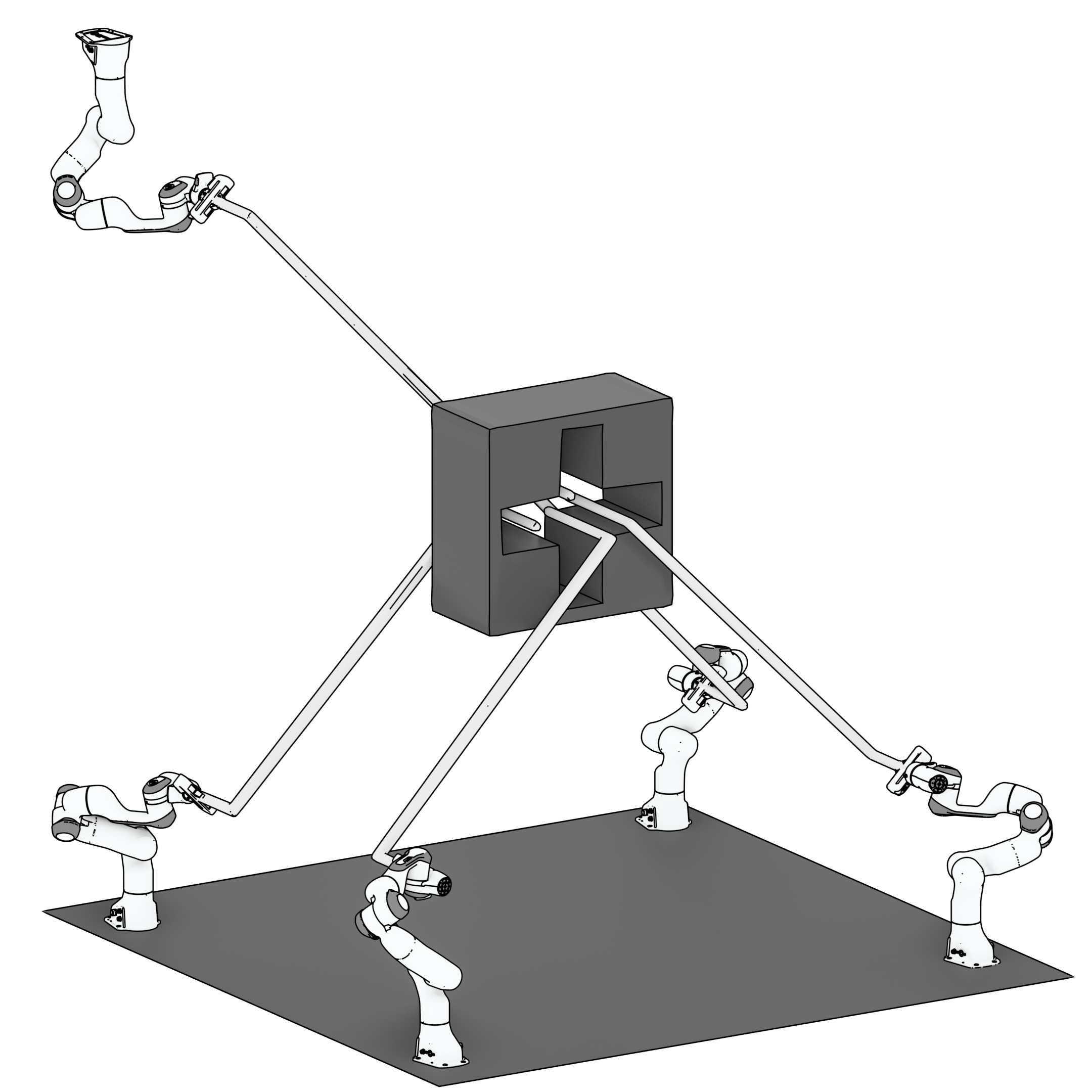}}&dRRT&-&T/O&-&T/O&-&T/O&-&-&-\\%
&\cellcolor{lightgray}CBSMP&\cellcolor{lightgray}-&\cellcolor{lightgray}6.960&\cellcolor{lightgray}274&\cellcolor{lightgray}23.347&\cellcolor{lightgray}745&\cellcolor{lightgray}T/O&\cellcolor{lightgray}1813&\cellcolor{lightgray}4.4\%&\cellcolor{lightgray}\textbf{65.6}\%\\%
\arrayrulecolor{darkgrey}\cline{2-11}\arrayrulecolor{black}%
&\multirow{4}{*}{\textcolor{mycolor}{\methodname}}&\textcolor{mycolor}{1}&\textcolor{mycolor}{54.344}&\textcolor{mycolor}{308}&\textcolor{mycolor}{T/O}&\textcolor{mycolor}{349}&\textcolor{mycolor}{T/O}&\textcolor{mycolor}{385}&\textcolor{mycolor}{53.4\%}&\textcolor{mycolor}{25.6\%}\\%
&&\textcolor{mycolor}{\cellcolor{lightgray}200}&\textcolor{mycolor}{\cellcolor{lightgray}26.911}&\textcolor{mycolor}{\cellcolor{lightgray}11}&\textcolor{mycolor}{\cellcolor{lightgray}T/O}&\textcolor{mycolor}{\cellcolor{lightgray}22}&\textcolor{mycolor}{\cellcolor{lightgray}T/O}&\textcolor{mycolor}{\cellcolor{lightgray}28}&\cellcolor{lightgray}97.7\%&\cellcolor{lightgray}30.0\%\\%
&&\textcolor{mycolor}{400}&\textcolor{mycolor}{28.761}&\textcolor{mycolor}{11}&\textcolor{mycolor}{T/O}&\textcolor{mycolor}{16}&\textcolor{mycolor}{T/O}&\textcolor{mycolor}{19}&\textcolor{mycolor}{96.0\%}&\textcolor{mycolor}{33.3\%}\\%
&&\textcolor{mycolor}{\cellcolor{lightgray}400*}&\textcolor{mycolor}{\cellcolor{lightgray}T/O}&\textcolor{mycolor}{\cellcolor{lightgray}11}&\textcolor{mycolor}{\cellcolor{lightgray}T/O}&\textcolor{mycolor}{\cellcolor{lightgray}14}&\cellcolor{lightgray}T/O&\cellcolor{lightgray}15&\cellcolor{lightgray}97.3\%&\cellcolor{lightgray}20.0\%\\%
\cline{1-11}%
\end{tabularx}%
\caption{We randomly generated 8 start-goal pairs per problem and tested each on three algorithms, running each pair 10 times with a 60s timeout, yielding 80 results per algorithm. We report the first quartile (Q1), median, and third quartile (Q3). Here, \textit{T} is total solve time, \textit{Q} is the number of path sets queried by the high-level planner, and \Weihang{\textit{Coord. ratio} is the ratio between the scheduling time and the total time.} *:As mentioned in~\cref{sec:retry}, we also define a maximum number of attempts before reconstructing the PRMs. By default, this maximum is 10, but for 400* it is increased to 20.}
\label{tab:all_arm}%
\end{table*}

%% file: includes/06_Conclusion.tex
\section{Conclusion}
This paper presented the \expandmethodname (\methodname) algorithm for multi-robot motion planning, demonstrating its effectiveness in manipulator environments with manifold constraints. With iterative path rescheduling and an informed path generation mechanism, \methodname efficiently reduces the complexity of finding collision-free solutions in shared, clustered workspaces with manifold constraints. \methodname balances the search between low-level planners and the high-level scheduler. This novel high-level scheduling allows \methodname to query significantly fewer paths from the low-level planner to find a valid solution. Future work includes automating the selection of rescheduling attempts, exploiting both CPU and GPU parallelism, and \Weihang{extending the method to optimal path planning}.

%% file: includes/07_Acknowledgements.tex
\section*{Acknowledgements}

 The authors would like to thank Dr. Gavin Britz from Methodist Hospital, Houston, TX, USA, for discussions that led to the definition of the problem considered in this paper.

%% file: references.bib
@article{lim2022cbsb,
  title={CBS-Budget (CBSB): A complete and bounded suboptimal search for multi-agent path finding},
  author={Lim, Jaein and Tsiotras, Panagiotis},
  journal={Artificial Intelligence},
  pages={104349},
  year={2025},
  publisher={Elsevier}
}

@inproceedings{barer2014ecbs,
  title={Suboptimal variants of the conflict-based search algorithm for the multi-agent pathfinding problem},
  author={Barer, Max and Sharon, Guni and Stern, Roni and Felner, Ariel},
  booktitle={Proceedings of the International Symposium on Combinatorial Search},
  volume={5},
  number={1},
  pages={19--27},
  year={2014}
}

@inproceedings{okumura2022ctrm,
  author = {Okumura, Keisuke and Yonetani, Ryo and Nishimura, Mai and Kanezaki, Asako},
  title = {CTRMs: Learning to Construct Cooperative Timed Roadmaps for Multi-Agent Path Planning in Continuous Spaces},
  year = {2022},
  isbn = {9781450392136},
  publisher = {International Foundation for Autonomous Agents and Multiagent Systems},
  booktitle = {Proceedings of the 21st International Conference on Autonomous Agents and Multiagent Systems},
  pages = {972–981},
  numpages = {10},
  series = {AAMAS '22}
}

@INPROCEEDINGS{pan2020augment_control,
  author={Pan, Tianyang and Verginis, Christos K. and Wells, Andrew M. and Kavraki, Lydia E. and Dimarogonas, Dimos V.},
  booktitle={2020 IEEE/RSJ International Conference on Intelligent Robots and Systems (IROS)}, 
  title={Augmenting Control Policies with Motion Planning for Robust and Safe Multi-robot Navigation}, 
  year={2020},
  volume={},
  number={},
  pages={6975-6981},
  doi={10.1109/IROS45743.2020.9341153}}

@ARTICLE{mcbeth2023topoguidance,
  author={McBeth, Courtney and Motes, James and Uwacu, Diane and Morales, Marco and Amato, Nancy M.},
  journal={IEEE Robotics and Automation Letters}, 
  title={Scalable Multi-Robot Motion Planning for Congested Environments With Topological Guidance}, 
  year={2023},
  volume={8},
  number={11},
  pages={6867-6874},
  doi={10.1109/LRA.2023.3312980}}

@article{kingston2019exploring-implicit-spaces-for-constrained,
  author = "Kingston, Zachary and Moll, Mark and Kavraki, Lydia E.",
  doi = "10.1177/0278364919868530",
  journal = "Intl.\ J.\ of Robotics Research",
  month = sep,
  number = "10--11",
  pages = "1151--1178",
  title = "Exploring Implicit Spaces for Constrained Sampling-Based Planning",
  volume = "38",
  year = "2019"
}

@INPROCEEDINGS{mujoco,
  author={Todorov, Emanuel and Erez, Tom and Tassa, Yuval},
  booktitle={2012 IEEE/RSJ International Conference on Intelligent Robots and Systems}, 
  title={MuJoCo: A physics engine for model-based control}, 
  year={2012},
  volume={},
  number={},
  pages={5026-5033},
  doi={10.1109/IROS.2012.6386109}}

@article{sucan2012the-open-motion-planning-library,
    Author = {Ioan A. {\c{S}}ucan and Mark Moll and Lydia E. Kavraki},
    Doi = {10.1109/MRA.2012.2205651},
    Journal = {{IEEE} Robotics \& Automation Magazine},
    Month = {12},
    Number = {4},
    Pages = {72--82},
    Title = {The {O}pen {M}otion {P}lanning {L}ibrary},
    Note = {\url{https://ompl.kavrakilab.org}},
    Volume = {19},
    Year = {2012}
}

@article{shaoul2024multi,
  title={Multi-robot motion planning with diffusion models},
  author={Shaoul, Yorai and Mishani, Itamar and Vats, Shivam and Li, Jiaoyang and Likhachev, Maxim},
  journal={arXiv preprint arXiv:2410.03072},
  year={2024}
}

@article{vcap2015prioritized,
  title={Prioritized planning algorithms for trajectory coordination of multiple mobile robots},
  author={{\v{C}}{\'a}p, Michal and Nov{\'a}k, Peter and Kleiner, Alexander and Seleck{\`y}, Martin},
  journal={IEEE Transactions on Automation Science and Engineering},
  volume={12},
  number={3},
  pages={835--849},
  year={2015},
  publisher={IEEE}
}

@article{zhang2024d,
  title={D-PBS: Dueling priority-based search for multiple nonholonomic robots motion planning in congested environments},
  author={Zhang, Xiaotong and Xiong, Gang and Wang, Yuanjing and Teng, Siyu and Chen, Long},
  journal={IEEE Robotics and Automation Letters},
  volume={9},
  number={7},
  pages={6288--6295},
  year={2024},
  publisher={IEEE}
}

@inproceedings{ma2019searching,
  title={Searching with consistent prioritization for multi-agent path finding},
  author={Ma, Hang and Harabor, Daniel and Stuckey, Peter J and Li, Jiaoyang and Koenig, Sven},
  booktitle={Proceedings of the AAAI conference on Artificial Intelligence},
  volume={33},
  number={01},
  pages={7643--7650},
  year={2019}
}

@article{orthey2024multilevel,
  title={Multilevel motion planning: A fiber bundle formulation},
  author={Orthey, Andreas and Akbar, Sohaib and Toussaint, Marc},
  journal={The international journal of robotics research},
  volume={43},
  number={1},
  pages={3--33},
  year={2024},
  publisher={SAGE Publications Sage UK: London, England}
}

@ARTICLE{solis2024arc,
  author={Solis, Irving and Motes, James and Qin, Mike and Morales, Marco and Amato, Nancy M.},
  journal={IEEE Robotics and Automation Letters}, 
  title={Adaptive Robot Coordination: A Subproblem-based Approach for Hybrid Multi-Robot Motion Planning}, 
  year={2024},
  volume={},
  number={},
  pages={1-8},
  doi={10.1109/LRA.2024.3420548}}

@inproceedings{van2011reciprocal,
  title={Reciprocal collision avoidance with acceleration-velocity obstacles},
  author={Van Den Berg, Jur and Snape, Jamie and Guy, Stephen J and Manocha, Dinesh},
  booktitle={2011 IEEE International Conference on Robotics and Automation},
  pages={3475--3482},
  year={2011},
  organization={IEEE}
}

@inproceedings{van2005prioritized,
  title={Prioritized motion planning for multiple robots},
  author={Van Den Berg, Jur P and Overmars, Mark H},
  booktitle={2005 IEEE/RSJ International Conference on Intelligent Robots and Systems},
  pages={430--435},
  year={2005},
  organization={IEEE}
}

@ARTICLE{kazumi2022continuous_prioritized_interval,
  author={Kasaura, Kazumi and Nishimura, Mai and Yonetani, Ryo},
  journal={IEEE Robotics and Automation Letters}, 
  title={Prioritized Safe Interval Path Planning for Multi-Agent Pathfinding With Continuous Time on 2D Roadmaps}, 
  year={2022},
  volume={7},
  number={4},
  pages={10494-10501},
  doi={10.1109/LRA.2022.3187265}}

@Inbook{khamis2015multirobot,
author="Khamis, Alaa
and Hussein, Ahmed
and Elmogy, Ahmed",
editor="Koub{\^a}a, Anis
and Mart{\'i}nez-de Dios, J.Ramiro",
title="Multi-robot Task Allocation: A Review of the State-of-the-Art",
bookTitle="Cooperative Robots and Sensor Networks 2015",
year="2015",
publisher="Springer International Publishing",
address="Cham",
pages="31--51",
isbn="978-3-319-18299-5",
doi="10.1007/978-3-319-18299-5_2",
url="https://doi.org/10.1007/978-3-319-18299-5_2"
}

@article{solis_vidana_representation-optimal_2021,
	title = {Representation-Optimal Multi-Robot Motion Planning Using Conflict-Based Search},
	volume = {6},
	rights = {https://creativecommons.org/licenses/by/4.0/legalcode},
	issn = {2377-3766, 2377-3774},
	url = {https://ieeexplore.ieee.org/document/9387143/},
	doi = {10.1109/LRA.2021.3068910},
	pages = {4608--4615},
	number = {3},
	journaltitle = {{IEEE} Robotics and Automation Letters},
	shortjournal = {{IEEE} Robot. Autom. Lett.},
	author = {Solis Vidana, Juan Irving and Motes, James and Sandstrom, Read and Amato, Nancy},
	urldate = {2024-06-07},
	date = {2021-07},
	langid = {english},
}

@article{Kant1986TowardET,
  title={Toward Efficient Trajectory Planning: The Path-Velocity Decomposition},
  author={Kamal Kant and Steven W. Zucker},
  journal={The International Journal of Robotics Research},
  year={1986},
  volume={5},
  pages={72 - 89},
  url={https://api.semanticscholar.org/CorpusID:120059732}
}

@ARTICLE{988973,
  author={Simeon, T. and Leroy, S. and Lauumond, J.-P.},
  journal={IEEE Transactions on Robotics and Automation}, 
  title={Path coordination for multiple mobile robots: a resolution-complete algorithm}, 
  year={2002},
  volume={18},
  number={1},
  pages={42-49},
  doi={10.1109/70.988973}}

@inproceedings{sanchez_using_2002,
  author={Sanchez, G. and Latombe, J.-C.},
  booktitle={Proceedings 2002 IEEE International Conference on Robotics and Automation (Cat. No.02CH37292)}, 
  title={Using a PRM planner to compare centralized and decoupled planning for multi-robot systems}, 
  year={2002},
  volume={2},
  number={},
  pages={2112-2119 vol.2},
  doi={10.1109/ROBOT.2002.1014852}}

@INPROCEEDINGS{kuffner2000rrtconnect,
  author={Kuffner, J.J. and LaValle, S.M.},
  booktitle={Proceedings 2000 ICRA. Millennium Conference. IEEE International Conference on Robotics and Automation. Symposia Proceedings (Cat. No.00CH37065)}, 
  title={RRT-connect: An efficient approach to single-query path planning}, 
  year={2000},
  volume={2},
  number={},
  pages={995-1001 vol.2},
  doi={10.1109/ROBOT.2000.844730}}

@article{cbs,
title = {Conflict-based search for optimal multi-agent pathfinding},
journal = {Artificial Intelligence},
volume = {219},
pages = {40-66},
year = {2015},
issn = {0004-3702},
doi = {https://doi.org/10.1016/j.artint.2014.11.006},
url = {https://www.sciencedirect.com/science/article/pii/S0004370214001386},
author = {Guni Sharon and Roni Stern and Ariel Felner and Nathan R. Sturtevant}
}

@article{Solovey2013FindingAN,
  title={Finding a needle in an exponential haystack: Discrete RRT for exploration of implicit roadmaps in multi-robot motion planning},
  author={Kiril Solovey and Oren Salzman and Dan Halperin},
  journal={The International Journal of Robotics Research},
  year={2013},
  volume={35},
  pages={501 - 513},
  url={https://api.semanticscholar.org/CorpusID:152780}
}

@article{shome_drrt_2020,
	title = {{dRRT}*: Scalable and informed asymptotically-optimal multi-robot motion planning},
	volume = {44},
	issn = {1573-7527},
	url = {https://doi.org/10.1007/s10514-019-09832-9},
	doi = {10.1007/s10514-019-09832-9},
	shorttitle = {{dRRT}*},
	pages = {443--467},
	number = {3},
	journaltitle = {Autonomous Robots},
	shortjournal = {Auton Robot},
	author = {Shome, Rahul and Solovey, Kiril and Dobson, Andrew and Halperin, Dan and Bekris, Kostas E.},
	urldate = {2024-06-07},
	date = {2020-03-01},
	langid = {english},
}

@ARTICLE{508439,
  author={Kavraki, L.E. and Svestka, P. and Latombe, J.-C. and Overmars, M.H.},
  journal={IEEE Transactions on Robotics and Automation}, 
  title={Probabilistic roadmaps for path planning in high-dimensional configuration spaces}, 
  year={1996},
  volume={12},
  number={4},
  pages={566-580},
  doi={10.1109/70.508439}}

@ARTICLE{li2021GNN,
  author={Li, Qingbiao and Lin, Weizhe and Liu, Zhe and Prorok, Amanda},
  journal={IEEE Robotics and Automation Letters}, 
  title={Message-Aware Graph Attention Networks for Large-Scale Multi-Robot Path Planning}, 
  year={2021},
  volume={6},
  number={3},
  pages={5533-5540},
  doi={10.1109/LRA.2021.3077863}}

@INPROCEEDINGS{arul2021vrvo,
  author={Arul, Senthil Hariharan and Manocha, Dinesh},
  booktitle={2021 IEEE/RSJ International Conference on Intelligent Robots and Systems (IROS)}, 
  title={V-RVO: Decentralized Multi-Agent Collision Avoidance using Voronoi Diagrams and Reciprocal Velocity Obstacles}, 
  year={2021},
  volume={},
  number={},
  pages={8097-8104},
  doi={10.1109/IROS51168.2021.9636618}}

@article{okumura2023lacam, title={LaCAM: Search-Based Algorithm for Quick Multi-Agent Pathfinding}, volume={37}, url={https://ojs.aaai.org/index.php/AAAI/article/view/26377}, DOI={10.1609/aaai.v37i10.26377}, number={10}, journal={Proceedings of the AAAI Conference on Artificial Intelligence}, author={Okumura, Keisuke}, year={2023}, month={6}, pages={11655-11662} }

@INPROCEEDINGS{li2022learningpc,
  author={Li, Wenhao and Chen, Hongjun and Jin, Bo and Tan, Wenzhe and Zha, Hongyuan and Wang, Xiangfeng},
  booktitle={2022 International Conference on Robotics and Automation (ICRA)}, 
  title={Multi-Agent Path Finding with Prioritized Communication Learning}, 
  year={2022},
  volume={},
  number={},
  pages={10695-10701},
  doi={10.1109/ICRA46639.2022.9811643}}

@INPROCEEDINGS{wiktor2014psw,
  author={Wiktor, Adam and Scobee, Dexter and Messenger, Sean and Clark, Christopher},
  booktitle={2014 IEEE/RSJ International Conference on Intelligent Robots and Systems}, 
  title={Decentralized and complete multi-robot motion planning in confined spaces}, 
  year={2014},
  volume={},
  number={},
  pages={1168-1175},
  doi={10.1109/IROS.2014.6942705}}

@inproceedings{okumura2023quick,
  title     = {Quick Multi-Robot Motion Planning by Combining Sampling and Search},
  author    = {Okumura, Keisuke and Défago, Xavier},
  booktitle = {Proceedings of the Thirty-Second International Joint Conference on Artificial Intelligence, {IJCAI-23}},
  pages     = {252--261},
  year      = {2023},
  month     = {8},
  doi       = {10.24963/ijcai.2023/29},
}

@article{kingston2018ar,
  author = {Kingston, Zachary and Moll, Mark and Kavraki, Lydia E.},
  title = {Sampling-Based Methods for Motion Planning with Constraints},
  journal = {Annual Review of Control, Robotics, and Autonomous Systems},
  year = {2018},
  volume = {1},
  number = {1},
  pages = {159--185},
  doi = {10.1146/annurev-control-060117-105226},
}
